\documentclass{article}

\usepackage[preprint]{neurips_2022}
\usepackage[utf8]{inputenc} 
\usepackage[T1]{fontenc}  
\usepackage{hyperref}      
\usepackage{url}           
\usepackage{booktabs}       
\usepackage{amsfonts}      
\usepackage{nicefrac}       
\usepackage{microtype}      
\usepackage{xcolor}         
\usepackage{graphicx}       
\usepackage{amssymb,amsmath,amsthm}        
\usepackage{bm,caption}				
\usepackage{subcaption}
\usepackage{multirow}

\hypersetup{
  colorlinks   = true,
  urlcolor     = blue,
  linkcolor    = blue,
  citecolor    = blue
}

\newcommand*{\smax}{\mathop{\stackrel{\mathrm{soft}}{\max}}}
\newcommand*{\T}{{\mathsf{T}}}

\newtheorem{theorem}{Theorem}[section]
\newtheorem{lemma}[theorem]{Lemma}
\newtheorem{corollary}[theorem]{Corollary}
\usepackage{algorithm,algorithmic}

\DeclareMathOperator{\sech}{sech}
\DeclareMathOperator{\relu}{ReLU}
\DeclareMathOperator{\rank}{rank}
\DeclareMathOperator{\std}{std}

\title{Standalone Neural ODEs with Sensitivity Analysis}

\author{
  Rym Jaroudi
  \thanks{
  Department of Science and Technology, Linköping University, SE-601\,74 Norrköping, Sweden.} \\
  \texttt{rym.jaroudi@liu.se} \\
  \And
  Lukáš Malý\footnotemark[1] \\
  \texttt{lukas.maly@liu.se} \\
  \And
  Gabriel Eilertsen\footnotemark[1] \\
  \texttt{gabriel.eilertsen@liu.se} \\
  \And
  B.~Tomas Johansson\footnotemark[1] \\
  \texttt{tomas.johansson@liu.se} \\
  \And
  Jonas Unger\footnotemark[1] \\
  \texttt{jonas.unger@liu.se} \\
  \And
  George Baravdish\footnotemark[1] \\
  \texttt{george.baravdish@liu.se} 
}

\begin{document}

\maketitle

\begin{abstract}
	This paper presents the Standalone Neural ODE (sNODE), a continuous-depth neural ODE model capable of describing a full deep neural network. This uses a novel nonlinear conjugate gradient (NCG) descent optimization scheme for training, where the Sobolev gradient can be incorporated to improve smoothness of model weights. We also present a general formulation of the neural sensitivity problem and show how it is used in the NCG training. The sensitivity analysis provides a reliable measure of uncertainty propagation throughout a network, and can be used to study model robustness and to generate adversarial attacks. Our evaluations demonstrate that our novel formulations lead to increased robustness and performance as compared to ResNet models, and that it opens up for new opportunities for designing and developing machine learning with improved explainability.
\end{abstract}

\section{Introduction}
\setcounter{footnote}{0}
Neural ordinary differential equations (NODEs),~\citep{Chen2018}, have enabled the formulation of residual networks with continuous depth, presenting advantages including memory efficiency during training and the possibility of trading precision against computational cost during inference. Recently, NODEs have been extended to include time/depth\footnote{We will use time and depth interchangeably throughout the paper.} dependence of weights using the Galerkin ODE-net ~\citep{Massaroli2020} and Stateful ODE-net~\citep{Queiruga2021}. However, until now, it has not been rigorously demonstrated how a full neural network can be modeled by a self-contained ODE model, i.e. NODEs has only been considered as an intermediate block embedded between ``conventional'' layers for feature extraction and classification, see Figure~\ref{fig:node}.

This paper presents a \emph{Standalone Neural ODE}, sNODE, which is a continuous-depth NODE model that describes full trainable networks in a self-contained fashion. We propose a mini-batch nonlinear conjugate gradient (NCG) optimization scheme with a general cost functional, which covers a wide variety of loss functions and penalty terms involving $L^p$ and Sobolev norms. A key benefit of a sNODE model is that the mathematical properties of the model holds true for the end-to-end mapping, from input data points to predictions. Hence, our novel sNODE formulation enables, e.g., analysis of the behavior of predictions under different perturbations of data points or weights, and general sensitivity analysis of the mapping. To demonstrate the benefits of the sNODE network, we formulate the neural sensitivity problem and analyze how uncertainty due to input noise is propagated through the network. We show how this also generalizes to traditional ResNets trained with stochastic gradient descent (SGD). Furthermore, we show how the neural sensitivity problem can be used to generate and analyze adversarial attacks. 
The contributions of the paper can be summarized as:

\begin{figure}
	\centering
	\includegraphics[width=0.9\linewidth,trim={0 8.8cm 0 0},clip]{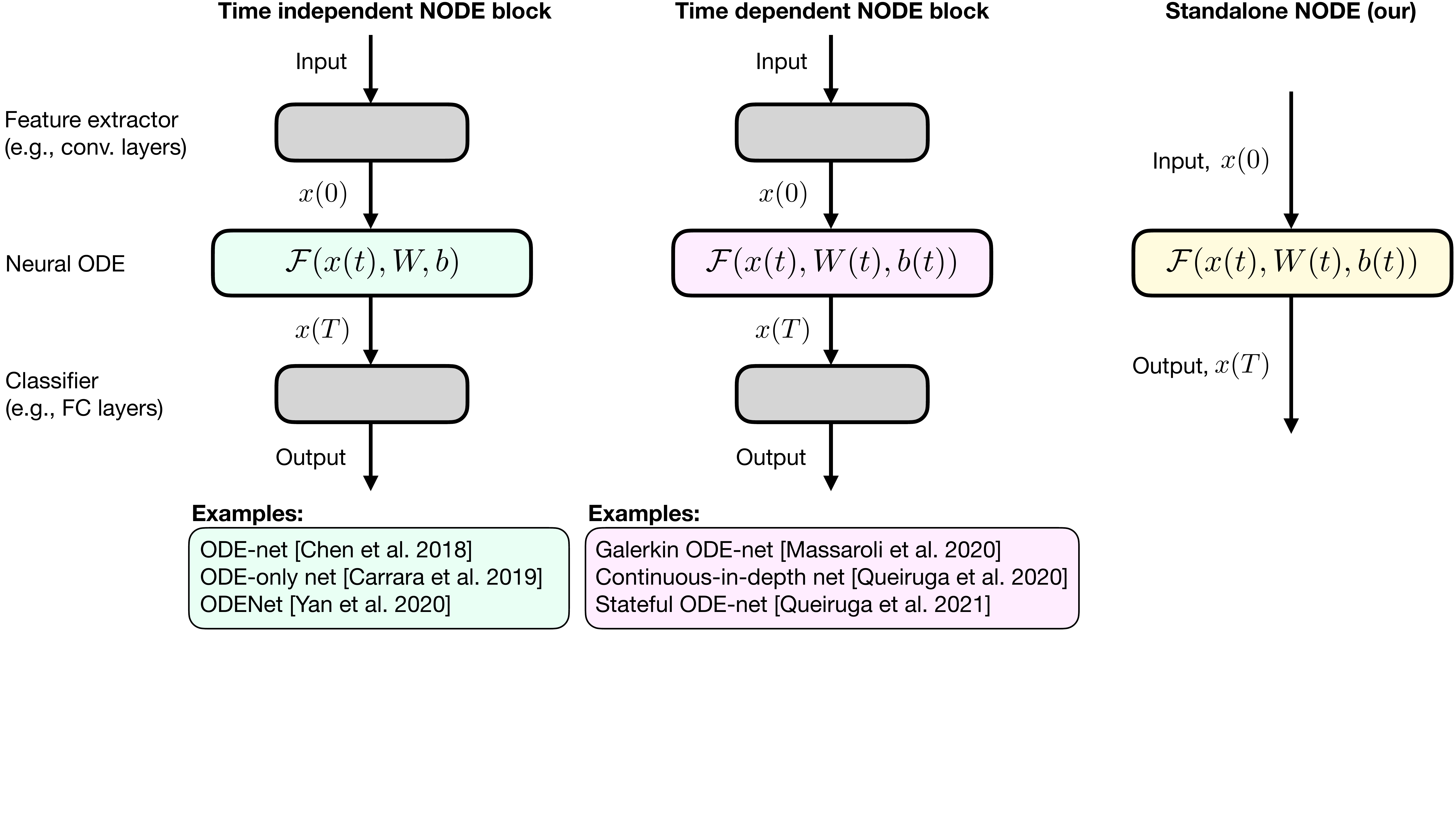}
	\caption{The original NODE formulation and some subsequent methods use the NODE as an intermediate block between feature extractor layers and a classifier, and is formulated without time dependence (left). Later work expand the NODE block to use time dependent parameterization (middle). We consider a Standalone neural ODE (right), which means that the NODE itself needs to perform feature extraction and classification mapping. Although more challenging, this allows for a theoretical formulation of the sensitivity of the network to input perturbations.}
	\label{fig:node}
\end{figure}

\begin{itemize}
	\item A novel NCG solver for Standalone depth-varying NODEs with a sensitivity analysis for determining the step length of the descent direction.
	\item An introduction of Sobolev regularization for maintaining accuracy when the level of noise is increasing in the input data. 
	\item A general sensitivity analysis formulation, which enables studying the robustness of trained weights, i.e., without requiring running inference with noisy data points. We show how the analysis generalizes beyond NODEs, e.g. to SGD-trained residual networks.
	\item A comparison of Euler discretization against an adaptive ODE solver, and a direct comparison of Euler NODEs to the equivalent SGD-trained residual networks.
	\item A neural sensitivity problem formulation for generating powerful adversarial examples for neural ODEs and equivalent SGD models.
\end{itemize}

Our analysis shows how sNODE models can be optimized using both Euler discretization and adaptive ODE solvers. This is an important contribution as it enables direct comparisons and provides a bridge between ODE-based models and ``conventional'' ResNets. This opens up for new opportunities for developing machine learning algorithms with theoretical foundation and explainable behavior.

\section{Related work}
\paragraph{Neural ODEs.}
NODEs were formally introduced by~\cite{Chen2018} as a representation of infinitely deep ResNet architectures. It has been further demonstrated that many of the networks commonly used in deep learning can be interpreted as different discretization schemes of differential equations~\citep{Lu2018}, but typically under the restriction that each layer is parameterized independently.
There is a range of follow-up work that consider different properties of NODEs. For example, it has been shown how the representations learned by NODEs preserve the topology of the input space, which can be a limiting factor in terms of model expressivity~\citep{Dupont2019}. To overcome this issue, the input space can be augmented to give the NODE a higher-dimensional space for finding more complex functions~\citep{Dupont2019}. It has also been demonstrated how NODEs offer an increased robustness with respect to noise perturbations and adversarial attacks~\citep{yan2019robustness}. NODEs can be further stabilized by injecting noise during training~\citep{Liu2019}. In the training of NODEs typically adjoint problems are used and these can be expensive in terms of computational usage, however, there is a recent study~\citep{matsubara2021symplectic} on efficient implementation of such adjoint problems.

The original work considers NODEs with no time/depth dependence, i.e., the same transformation is applied for each point in time. Other works have suggested formulating NODEs with time-dependence, using the Galerkin ODE-net~\citep{Massaroli2020} or the recent Stateful ODE-net~\citep{Queiruga2021}. However, until now NODEs have been considered an intermediate network component, which is embedded between conventional neural network layers, as illustrated in Figure~\ref{fig:node}. We focus on the problem of using NODEs in isolation, describing the end-to-end mapping of a neural network. The benefit of this formulation is that the mathematical properties of the model holds true, from input data points to predictions. This makes it possible to, e.g., analyze the behavior of predictions under different perturbations of data points or weights. As we will demonstrate, this enables a general sensitivity analysis of the mapping.
There is a recent example where NODEs were used in isolation from conventional network layers~\citep{baravdish2022learning} but only for small networks on 2D toy datasets.  
We substantially scale the size of the NODE network, and apply it to image input data. 

\paragraph{Sensitivity analysis.}
There are different notions of sensitivity within the neural network literature. Methods for explainable AI can be formulated for generating images that maximally activate some targeted layer or neuron~\citep{Simonyan2013,Zeiler2014,Yosinski2015}. Other methods investigate the sensitivity of the network output to changes in input~\citep{novak2018sensitivity,Pizarroso2020}. For example,
\cite{novak2018sensitivity} used the norm of the network’s Jacobian matrix as a sensitivity metric to conduct an empirical study and showed correlation between generalization and sensitivity in neural networks. There is also a range of previous work aimed towards finding the smallest changes in input images which lead to mispredictions, so called adversarial attacks~\citep{szegedy2013intriguing,goodfellow2014explaining,madry2018towards}. 
\cite{yan2019robustness} showed that ODE-Nets are more robust  than classical CNN models against both random Gaussian perturbations and adversarial attacks. \cite{carrara2019robustness} also studied the robustness of NODEs to adversarial attacks and arrived at the same conclusion. A way of defending against adversarial attacks, utilizing NODEs, is presented in~\citep{kang2021stable}.
In this paper, we analytically estimate a trained networks sensitivity to input perturbations without running perturbed data through the network. This is possible due to the Standalone NODE formulation. We show also how the sensitivity problem can be used for generating powerful adversarial attacks, as it allows us to find a direction for perturbation of the input that affects the output of the network with the greatest impact.

\paragraph{CG and Sobolev gradient.}
In the discrete case, training via the conjugate gradient (CG) has been tested earlier~\citep{bottou2018optimization}, and discrete Sobolev norms for neural network inputs and outputs have been investigated~\citep{Czarnecki2017sobolev}.  
We present a significantly different approach, where we apply NCG for minimizing a functional in the Lagrangian formulation in the continuous setting, with gradient descent obtained by a neural sensitivity problem. This optimization method is further extended to include the Sobolev gradient for trainable weights. Such problem formulation has only been considered very recently in~\citep{baravdish2022learning}, and we significantly expand on this.

\section{Standalone Neural ODEs}
The key goal of sNODE is to enable fully self-contained depth-varying neural ODEs capable of describing a full neural network as a standalone entity. 
The forward propagation of the input $\bm{x}^0\in\mathbb{R}^N$ through this network of infinite layers is given by
\begin{equation}
	\label{eq:xDE}
	\left\{
	\begin{aligned}
		\bm{x}'(t) & = \sigma(W(t)\bm{x}(t)+b(t)), \quad t \in I=(0,T) \\
		\bm{x}(0)  & = \bm{x}^0\quad\text{(given input layer)}         
	\end{aligned}
	\right.
\end{equation}
 
where $\sigma$ is a nonlinear activation function applied component-wise describing the activation of neurons and the affine transformations are represented by their layer-dependent weights, $W(t) \in \mathbb{R}^{N\times N}$ and biases, $b(t) \in \mathbb{R}^N$. The output layer is obtained at time $T$. 
In the numerical experiments, we will consider two commonly used nonlinear activation functions: the Rectified Linear Unit $\sigma(r) = \relu(r)$ with its derivative $\sigma'(r) = \smash{\chi_{\mathbb{R^{N}_{+}}}(r)}$ and the Hyperbolic Tangent $\sigma(r) =  \tanh(r)$ having as derivative $\sigma'(r) = \sech^{2}(r)$. 

The effect of perturbation of parameters and initial condition on the solution of~\eqref{eq:xDE} is described by
\begin{lemma}\label{lemmaSens}
	Assume that the solution $\bm{x}(t)$ of~\eqref{eq:xDE} is perturbed by $\varepsilon \bm{\xi}(t)$ when the parameters $W(t)$ and $b(t)$ are perturbed by $\varepsilon \omega(t)$ and $\varepsilon \beta(t)$ respectively (and when $\bm{x}^0$ is perturbed by $\varepsilon \bm{\xi}^0$), where  $\varepsilon>0$ is a small number. Then, the perturbation $\bm{\xi}(t)$ satisfies the linear equation given by \begin{equation}
	\left\{
	\begin{aligned}
		\bm{\xi}'(t) & = \sigma'\bigl(W(t) \bm{x}(t) + b(t)\bigr) \odot \bigl( \omega(t) \bm{x}(t) + \beta(t) + W(t) \bm{\xi}(t) \bigr), & t\in I, \\
		\bm{\xi}(0) & = \bm{\xi}^0,
	\end{aligned}
	\right.  
	\label{eq:sp}
	\end{equation}
	where $\sigma'$ is the derivative of $\sigma$. The \textbf{neural sensitivity problem} given by equation~\eqref{eq:sp} describes the propagation of perturbations of the parameters and the initial condition through the network. 
\end{lemma}
The neural sensitivity problem provides information about dependence of output data on input and parameters.
Solving the neural sensitivity problem~\eqref{eq:sp} for $\omega(t)=\beta(t)=0$ describes the propagation of the perturbation along the layers when the input is perturbed by $\bm{\xi}^0$.   
This can be used for studying the robustness of neural networks and paving the way for explainability.
The solution to the neural sensitivity problem~\eqref{eq:sp} for $\bm{\xi}^0=0$, $\omega(t)\neq0$ and $\beta(t)\neq0$ will be used in the training procedure to determine the step size of the descent direction in the NCG. 

\section{Training sNODEs via a Nonlinear Conjugate Gradient method}\label{sec:ncg}
In the continuous setting, deep learning is interpreted as a parameter estimation problem where the weights and biases have to be identified \citep{haber2017stable}; therefore, given a batch of $K$ data points, each with input $\bm{x}_k\in\mathbb{R}^N$ and desired output $\bm{y}_k\in\mathbb{R}^N$, $k=1,2,\ldots, K$, the training procedure can be cast into minimizing the following cost functional 
\begin{align}
	\label{eq:E}
	E(W,b) & = \frac{1}{K} \sum_{k=1}^K \Biggl(\frac{\mu_1}{2} \| \bm{x}_k(T) - \bm{y}_k \|_{\ell^2}^2+ \mu_2\,H\bigl(\bm{y}_k, \smax (\bm{x}_k(T))\bigr) + \frac{\mu_3}{2} \| \bm{x}_k(T)  \|_{\ell^2}^2 \Biggr) \\ & \quad + \frac{\mu_4}{2}\bigl(\|W\|_{L^2}^2 + \|b\|^2_{L^2}\bigr)  \nonumber
\end{align}
subject to  
\begin{equation}
	\label{eq:xDE2}
	\left\{
	\begin{aligned}
		\bm{x}_k'(t) & = \sigma(W(t) \bm{x}_k(t)+ \bm{b}(t)), \quad t \in I := (0,T) \\
		\bm{x}_k(0)  & = \bm{x}_k^0                                                  
	\end{aligned}
	\right.
\end{equation}
with $\bm{x}_k(T)$ the output layer, and the layer-dependent weight-matrix $W$ and bias $b$ belong to a suitable class of functions (e.g a Hilbert space). Here, $H(\bm{p}, \bm{q}) = - \sum_i p_i \log(q_i)$ is the cross-entropy function for probability vectors $\bm{p}, \bm{q} \in \mathbb{R}^N$ with $\smax(\bm{x})= \frac{\exp(x_i)}{\sum_{j} \exp(x_j)}$  measuring the quality of the predicted class label probabilities.

The choice of the loss function is critical in defining the outputs in a way that is suitable for the application at hand. Cross-entropy (CE) is the de facto standard in classification tasks. However, it has been demonstrated that $\ell^2$ can produce results on par or even better than CE~\citep{Hui2021evaluation}. For this reason, it is interesting to compare the two different loss functionals, as they could have very different impact on the sensitivity of a trained model.
The cost functional (\ref{eq:E}) covers commonly used ones in deep learning. For $\mu_1=0$, $\mu_2>0$, we have the cross-entropy loss and $\mu_3>0$ is also used to explicitly constrain the output to be small. This regularization against explosive growth in the final layer may yield accuracy improvements comparable to
dropout \citep{dauphin2020deconstructing}. For $\mu_1>0$ and $\mu_2=\mu_3=0$, we have the $\ell^2$-loss. Finally, $\mu_4>0$ is used to include Tikhonov regularization, also referred to as weight decay. 

We use the Lagrange multiplier method to rewrite the constrained minimization problem as follows
\begin{align}
	\mathcal{E}(W,b) & =\frac{1}{K} \sum_{k=1}^K \Biggl(\frac{\mu_1}{2} \| \bm{x}_k(T) - \bm{y}_k \|_{\ell^2}^2+ \mu_2\,H\bigl(\bm{y}_k, \smax (\bm{x}_k(T))\bigr) + \frac{\mu_3}{2} \| \bm{x}_k(T)  \|_{\ell^2}^2 \Biggr) \\ & \quad + \frac{\mu_4}{2}\bigl(\|W\|_{L^2}^2 + \|b\|^2_{L^2}\bigr) + \frac{1}{K} \sum_{k=1}^K \int_{0}^{T} \bigl\langle \bm{\lambda}_{k}(t), \sigma(W(t)\bm{x}_k(t)+b(t))-\bm{x}_k'(t) \bigr\rangle\,dt, \nonumber
\end{align}
where $\bm{\lambda}_k(t)$ is the Lagrange multiplier.

We solve the minimization problem using a conjugate gradient based iterative method proposed by \cite{baravdish2022learning}, where the descent direction, involving the gradient of $E(W,b)$, are calculated via the adjoint problem (see Lemma~\ref{lemmaAdj} below) and the learning rate is determined by the neural sensitivity problem~\eqref{eq:sp} with $\bm{\xi}^0=0$. 
\begin{lemma}
	\label{lemmaAdj}
	The Fréchet derivative of $\mathcal{E}(W,b)$ with respect to the $L^2$-inner product is given by 
	\begin{align}
		\label{eq:E'W}
		\mathcal{E}'_{\omega} & = \mu_4 W + \frac{1}{K} \sum_{k=1}^K \bm{x}_k^\T (\sigma'(W \bm{x}_k +b) \odot \bm{\lambda}_{k}) \\
		\label{eq:E'b}
		\mathcal{E}'_{\beta}  & = \mu_4 b + \frac{1}{K} \sum_{k=1}^K (\sigma'(W \bm{x}_k +b) \odot \bm{\lambda}_{k})             
	\end{align}
	where the Lagrange multiplier $\bm{\lambda}_k(t)$ is the solution to the \textbf{adjoint problem} given by
	\begin{equation}
		\label{eq:aDE2}
		\left\{
		\begin{aligned}
			\bm{\lambda}_k'(t) & = - W(t)^\T \Bigl( \sigma'\bigl( W(t) \bm{x}_k(t) + b(t) \bigr) \odot \bm{\lambda}_k(t)\Bigr), \\
			\bm{\lambda}_k(T)  & = \mu_1(\bm{x}_k(T) - \bm{y}_k)+ \mu_2 (\smax(\bm{x}_k(T))-\bm{y}_k)+\mu_3 \bm{x}_k(T).        
		\end{aligned}
		\right.
	\end{equation}
\end{lemma}

The gradient descent direction given by the Fréchet derivative as obtained in Lemma~\ref{lemmaAdj} is the steepest one with respect to the $L^2$-inner product on the interval $(0, T)$. For regularization purposes, one may consider the descent direction that is steepest with respect to the Sobolev $W^{1,2}$-inner product obtained by the transformation~\eqref{eq:S-xform}, which results in smoother trained weights, see Appendix~\ref{app:ncg}.  

\begin{corollary}
	\label{cor:SobDer}	
	The Fréchet derivative of $\mathcal{E}(W,b)$ with respect to the Sobolev $W^{1,2}$-inner product is obtained by applying the transformation $\mathcal{S}$ defined by
	\begin{equation}
		\label{eq:S-xform}
		\mathcal{S}u(t) = \frac{\cosh(T-t)}{\sinh T} \int_0^t u(s) \cosh(s)\,ds + \frac{\cosh t}{\sinh T} \int_t^T u(s) \cosh(T-s)\,ds, \quad t\in [0, T],
	\end{equation}
	to each component $u \in L^2(0,T)$ of $\mathcal{E}'_{\omega}$ and $\mathcal{E}'_{\beta}$ in~\eqref{eq:E'W} and~\eqref{eq:E'b}, respectively.
\end{corollary}

The proofs of Lemma~\ref{lemmaAdj} and Corollary~\ref{cor:SobDer} are included in Appendix~\ref{app:proofs}. The NCG algorithm together with details for computing the learning rate are in Appendix~\ref{app:ncg}.

\section{Numerical experiments} 
We perform experiments on the  MNIST~\citep{lecun1998mnist} dataset, which is composed of images of hand-written digits from~0 to~9, with a total of 60K training images and 10K testing images. We linearly resample the images to 14×14 pixel resolution and then vectorize them to a $196$-dimensional input.
With the intention of enabling comparisons of the continuous setting, i.e., sNODEs with the NCG training method, to the existing discrete setting, i.e., ResNets with standard stochastic gradient descent (SGD), we consider sNODEs on the interval $I=(0,T)$ of length $T = 3$ and ResNets composed of $150$ layers corresponding to the discretization level used in the sNODEs. The activation functions are scaled by the factor $1/50$, to have an exact copy of an sNODE solved with Euler discretization. The hidden layers have the same width as the number of pixels in the input images, i.e., $196$ for all models. In the final layer, only $10$ dimensions are used for predicting the $10$ classes.

To train the sNODEs, we apply steps of the NCG algorithm with data batches comprising of $K = 100$ images. Initially, normally distributed weights and bias functions $W^0(t)$ and $b^0(t)$ with mean $0$ and standard deviation $0.001$ are introduced. Each batch was used in $6$ iterations of the NCG. When a given set of data points $\bm{x}_k^0$ is to be classified, it is fed into the network, i.e., sNODE with reconstructed weights and biases, via the initial condition of the direct problem~\eqref{eq:xDE}. The final value $\bm{x}_k(T)$ of the solution to the differential equation is then used to determine the class of the respective point. For training of the discrete model, SGD is used through the RMSprop optimizer~\citep{hinton2012neural} with an initial learning rate of $0.01$. Apart from this, the setup is exactly replicating the NCG-trained sNODE, with the same initialization, batch size, training iterations, loss function, etc.

First, we show that, as a numerical scheme of solving sNODEs, Euler discretization outperforms ODE45 solvers besides allowing for a direct comparison between sNODE models and ResNets. Next, we study the performance and robustness capabilities of NCG-trained sNODEs for tanh/ReLU activation functions, $\ell^2$/CE-loss with/without weight decay, in comparison to SGD-trained ResNets, and we evaluate the ability of the neural sensitivity problem of measuring the perturbation propagation throughout the network.
Finally, we show how the neural sensitivity problem can be used both to study the robustness of the trained models and to generate adversarial attacks.
The NCG trainings as well as the sensitivity analysis were implemented in MATLAB, while SGD experiments were performed through Keras/TensorFlow and trained on an Nvidia Qudaro RTX 5000. Adversarial attacks for comparisons have been produced through the Adversarial Robustness Toolbox\footnote{\url{https://adversarial-robustness-toolbox.readthedocs.io}}.

\textbf{Euler discretization performs well in comparison to an adaptive ODE solver.}
In this experiment, we compare the Euler numerical scheme against an adaptive ODE solver (ODE45) for the discretization of sNODE when NCG training is considered on (14×14) MNIST. We note that the ODE45 solver was previously tested by \cite{baravdish2022learning} in the same setting of training sNODE via NCG; however, experiments were limited to 2D toy datasets. 
We choose tanh activation functions due to their smoothness property, and set $\mu_1=1$, $\mu_2=\mu_3=\mu_4=0$ in~\eqref{eq:E} hence minimizing the $\ell^2$-loss without weight decay. We perform $5$ separate training runs, each over $20$ epochs, and report the mean and standard deviation of training and testing accuracies during training in Figure~\ref{fig:mix_ode_eul}. 
We also test the effect of running inference through Euler discretized models with weights and biases learned through NCG where ODE45 solvers were used, and vice versa (see Table~\ref{tab:mix_ode_eul}).

\begin{figure}[h]
	\begin{subfigure}[b]{0.45\textwidth}
		\centering
		\includegraphics[width=\textwidth]{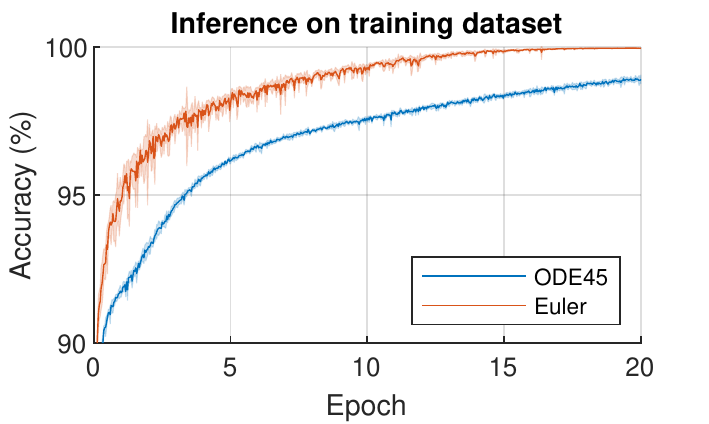}
	\end{subfigure}
	\hfill
	\begin{subfigure}[b]{0.45\textwidth}
		\centering
		\includegraphics[width=\textwidth]{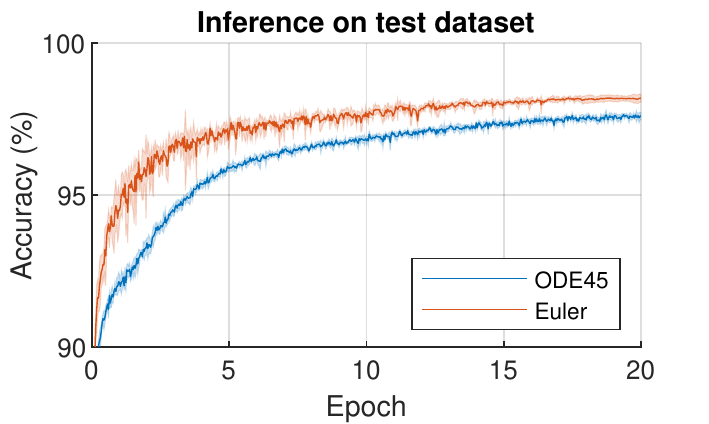}
	\end{subfigure}
	\caption{Mean and standard deviation of training and test accuracies over $5$ separate runs of sNODEs with NCG when using Euler discretization compared to the ODE45 solver.}
	\label{fig:mix_ode_eul}
	\vspace{-5mm}
\end{figure}

As reported in Figure~\ref{fig:mix_ode_eul} and Table~\ref{tab:mix_ode_eul}, using the Euler discretization rather than the ODE45 scheme may improve the model capabilities in terms of optimization via NCG and performance. Also, this setting makes it possible to perform a direct comparison between NODEs and SGD-trained residual networks, where the number of layers is set in accordance with the Euler discretization step.

\begin{table}[h]
	\vspace{-3mm}
	\caption{Mean and standard deviation of training and test accuracies ($\%$) over 5 separate runs on (14×14) MNIST for 20 epochs when Euler discretization and/or the ODE45 solver are considered.}
	\label{tab:mix_ode_eul}
	\begin{center}
		\begin{tabular}{l|l|l|l|l} 
			Training          & Euler          & Euler          & ODE45          & ODE45          \\
			Inference         & Euler          & ODE45          & ODE45          & Euler          \\
			\hline
			Training Accuracy & $99.94\pm0.01$ & $98.99\pm0.25$ & $98.90\pm0.15$ & $98.99\pm0.11$ \\ 
			Test Accuracy &  $98.19\pm0.13$ & $97.97\pm0.26$   & 
			$97.60\pm0.12$ & $97.64\pm0.10$ \\ 
		\end{tabular}
	\end{center}
\end{table}

\textbf{NCG-trained Euler sNODE models have the same or better accuracy and robustness when compared to SGD-trained ResNets.} We investigate the generalization and robustness capabilities of NCG-trained sNODEs in comparison to their equivalent SGD-trained ResNets under different settings. For the optimization, we compared minimizing CE-loss with $\mu_2=1$, $\mu_3=0.4$ and $\mu_1=0$ in~\eqref{eq:E} to $\ell^2$-loss obtained by $\mu_1=\mu_3=0$ and $\mu_2=1$. We also inspected the effect of regularization through weight decay by considering $\mu_4=0$ and $\mu_4=10^{-8}$. For each combination of activation and cost, we used NCG for sNODEs and SGD for ResNets to carry out $5$ separate trainings over $10$ epochs for each model, then we ran inference on clean datasets and noisy datasets obtained from perturbing the clean ones by Gaussian noises of different standard deviations $\{0.01, 0.03, 0.05, 0.08, 0.1, 0.13, 0.15, 0.18, 0.2\}$, and computed the accuracy of the classification. 
The models performance (classification accuracy on clean data) and robustness (classification accuracy on noisy data) are indicated respectively in Table~\ref{tab:all_acc} and Figure~\ref{fig:all_acc_nois}. 
\begin{table}[ht]
	\vspace{-5mm}
	\caption{Mean and standard deviation of classification accuracies ($\%$) on clean data of NCG-trained sNODEs and SGD-trained ResNets with different activations, loss and regularizations.}
	\label{tab:all_acc}
	\begin{center}
		\begin{tabular}{lcccc} 
			&\multicolumn{2}{c} {\textbf{Cross-entropy loss}} &\multicolumn{2}{c}{\textbf{$\ell^2$-loss}} \\ 
			         & ReLU           & tanh           & ReLU           & tanh           \\ 
			\hline 
			SGD      & $96.30\pm0.37$ & $95.21\pm0.26$ & $97.92\pm0.14$ & $97.92\pm0.05$ \\  
			SGD w.d. & $96.05\pm0.29$ & $95.40\pm0.32$ & $98.03\pm0.20$ & $97.82\pm0.15$ \\ 
			\hline
			NCG      & $98.05\pm0.10$ & $96.99\pm0.25$ & $98.03\pm0.22$ & $97.80\pm0.06$ \\  
			NCG w.d. & $97.60\pm0.08$ & $97.67\pm0.14$ & $97.73\pm0.08$ & $97.83\pm0.05$ \\ 
			NCG Sob. & $95.84\pm0.15$ & $97.13\pm0.06$ & $96.32\pm0.64$ & $96.93\pm0.06$ 
		\end{tabular}
	\end{center}
\end{table}

NCG-trained sNODEs outperform SGD-trained ResNets when CE-loss is considered whether weight decay is used or not, and they have similar robustness without weight decay. With $\ell^2$-loss, both have similar performance but with NCG sNODE being more robust when no weight decay is used. 
Using $\ell^2$-loss improves the accuracy of SGD-trained ResNets while worsening their robustness; however, it has no significant effect on the performance nor the robustness of NCG sNODEs. 
For each model, weight decay has no significant effect on its performance; however, it strongly decreases the robustness of NCG sNODEs. Employing Sobolev gradient descent in NCG training of sNODEs with tanh activation functions produces the most robust models.
Finally, by examining the magnitude of $\| \bm{x}(t) \|$, we note that the value of $\mu_3=0.4$ used in the experiments might not be optimal for Sobolev gradient descent with CE loss. A hyper-parameter search could potentially improve on this.

\begin{figure}[t]
	\begin{centering}
		\includegraphics[width=0.85\linewidth]{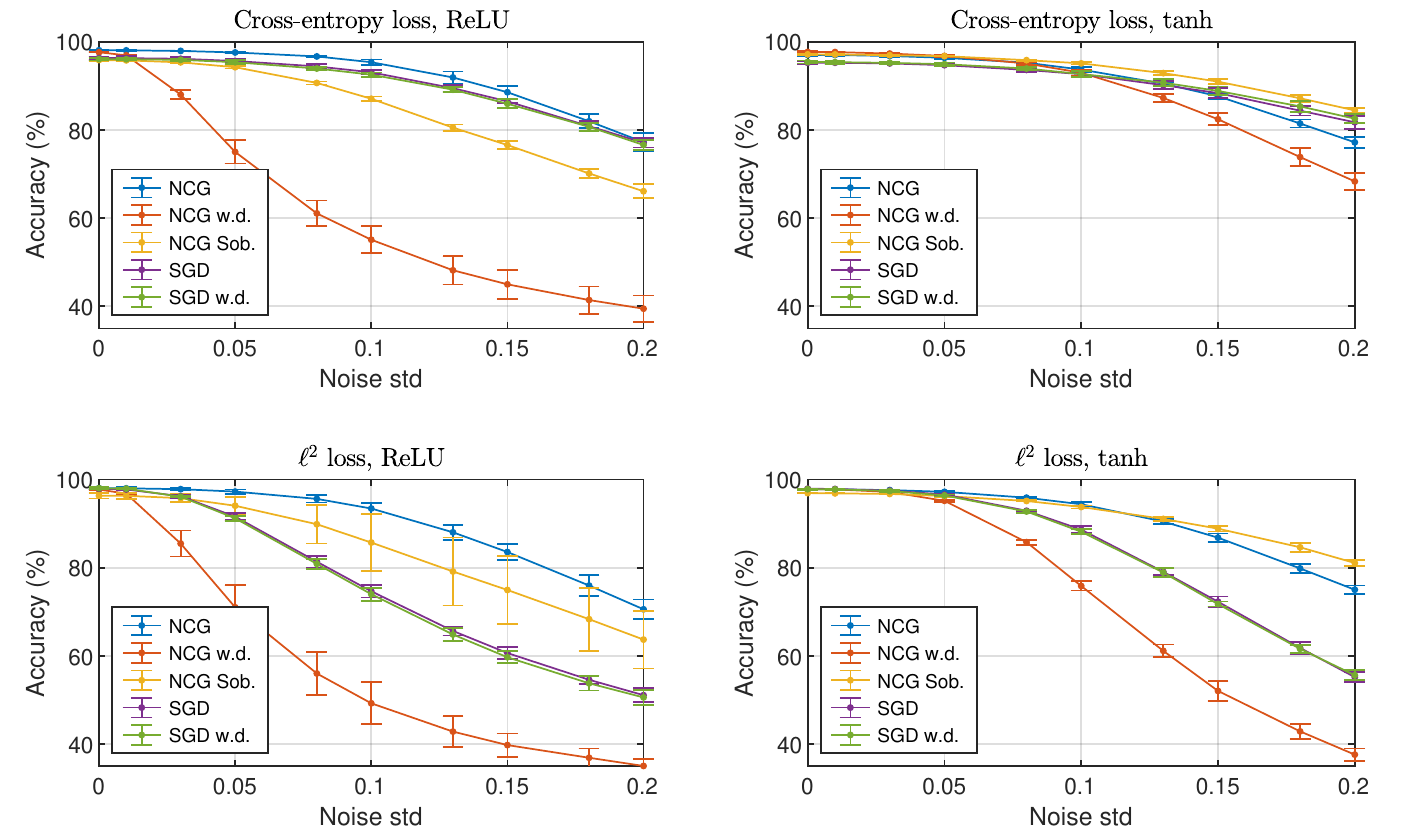}
		\caption{Robustness of different models and training settings. The plots show means and standard deviations of classification accuracies over 5 runs, on test data with varying amount of noise applied.}
		\label{fig:all_acc_nois}
	\end{centering}
	\vspace{-0.7cm}
\end{figure}

\textbf{The general sensitivity problem provides a good approximate description of noise propagation.}
We examine how well the solution $\bm{\xi}(t)$ of the neural sensitivity problem~\eqref{eq:sp} for $\omega(t)=\beta(t)=0$ and $\bm{\xi}^0$ a Gaussian noise with standard deviation $0.01$, estimates the measured perturbation $\bm{\delta}(t)$ throughout the network, with $\bm{\delta}(t):=\widetilde{\bm{x}}(t)-\bm{x}(t)$, where $\widetilde{\bm{x}}(t)$ and $\bm{x}(t)$ are the respective solutions to~\eqref{eq:xDE} with a noisy input $\bm{x}^0+\bm{\xi}^0$ and a clean input $\bm{x}^0$. We consider weights and biases obtained from the previous experiment to cover different settings: NCG-trained sNODEs/SGD-trained ResNets, tanh/ReLU activation functions, $\ell^2$/CE-loss, with and without weight decay, with and without Sobolev gradient descent. The estimated perturbation $\lVert \bm{\xi}(t) \rVert / \lVert \bm{x}(t) \rVert$ and the measured perturbation $\lVert \bm{\delta}(t) \rVert / \lVert \bm{x}(t) \rVert$ throughout the NCG-trained sNODEs and the SGD-trained ResNets with $\ell^2$-loss are presented in Figures~\ref{fig:sp_xi_delta_l2}. Similar experiments for the CE-loss setting, together with relative estimation errors $\lVert \bm{\xi}(t)-\bm{\delta}(t) \rVert / \lVert \bm{\xi}(t) \rVert$ are presented in Appendix~\ref{app:sensitivity}.  

\begin{figure}[ht]
	\begin{centering}
		\begin{subfigure}[b]{0.4\textwidth}
			\centering
			\includegraphics[width=\textwidth]{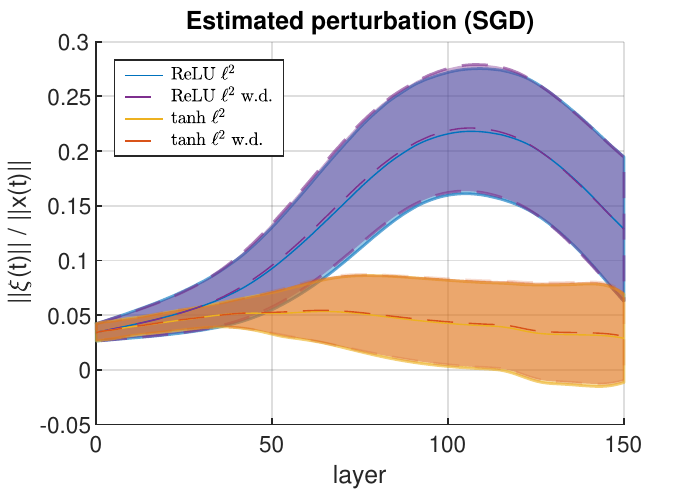}
		\end{subfigure}
		\begin{subfigure}[b]{0.4\textwidth}
			\centering
			\includegraphics[width=\textwidth]{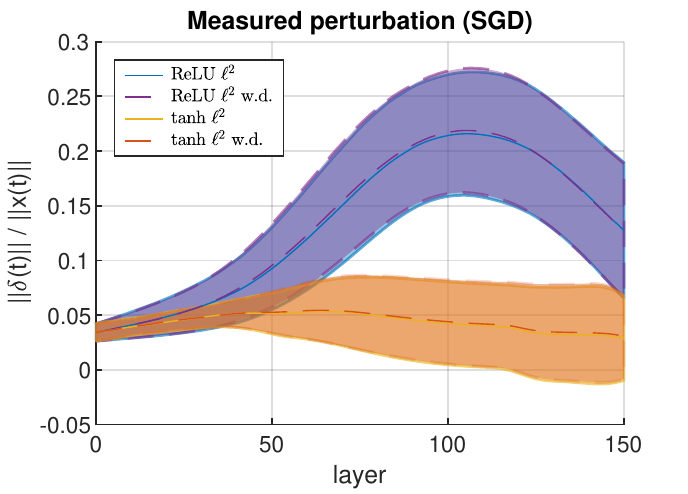}
		\end{subfigure}
		\begin{subfigure}[b]{0.4\textwidth}
			\centering
			\includegraphics[width=\textwidth]{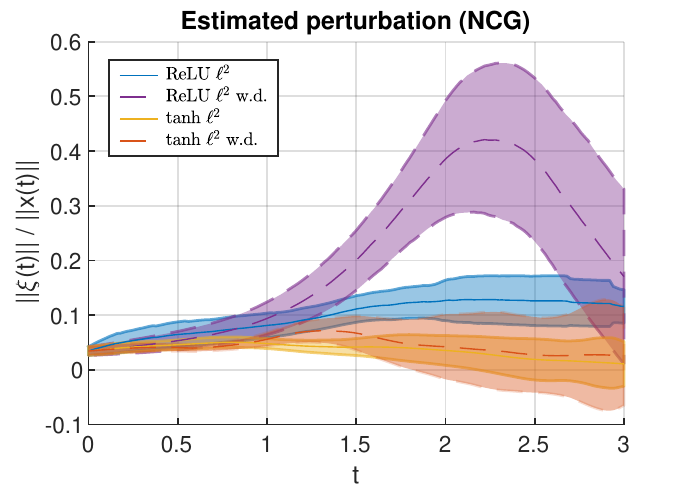}
		\end{subfigure}
		\begin{subfigure}[b]{0.4\textwidth}
			\centering
			\includegraphics[width=\textwidth]{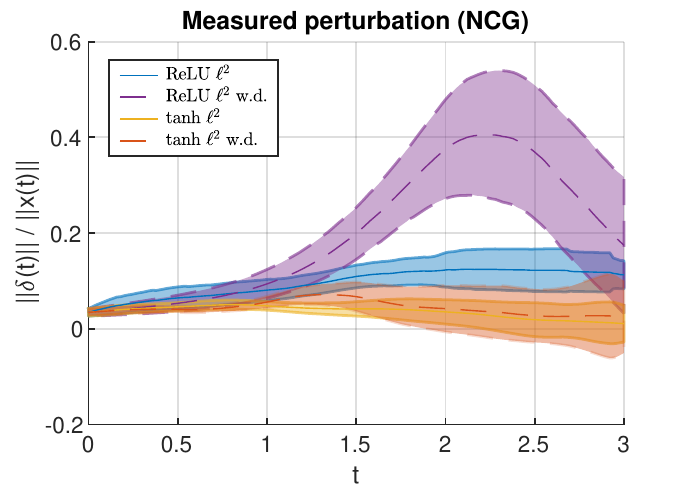}
		\end{subfigure}
		\caption{Evaluation of the neural sensitivity problem. Estimated propagation of the perturbation (left panel) compared to the measured perturbation (right panel) when $\std(\bm{\xi}^0)=\std(\bm{\delta}^0)=0.01$ through different SGD-trained ResNets (top row) and NCG-trained sNODEs (bottom row) using $\ell^2$-loss. }
		\label{fig:sp_xi_delta_l2}
	\end{centering}
	\vspace{-5mm}
\end{figure}

We note that weight decay has a significant effect on the sensitivity of NCG-trained sNODEs when tanh is used, i.e.,  perturbations are more controlled. This is in contrast to SGD-trained ResNets, where weight decay has no effect on the sensitivity. 
The obtained results corroborated well with what could be expected theoretically, and indicate that the neural sensitivity problem can be used to measure the propagation of a perturbation throughout NCG-trained sNODEs as well as SGD-trained ResNets with different activations and loss functions without running inference with noisy data. 

\textbf{The neural sensitivity problem can be used to generate powerful adversarial attacks.}
Given an input data point $\bm{x}^0 \in \mathbb{R}^N$, one infers its class by solving the sNODE~\eqref{eq:xDE} and determining which component of $\bm{x}(T)$ that has the greatest value. In adversarial attacks, an ``imperceptible error'' in the input leads to a significant change in the output. Such an attack can be found by an iterative process, which produces a chain of input datapoints $\bm{x}^1, \bm{x}^2, \ldots, \bm{x}^L$ in the vicinity of $\bm{x}^0$ so that a different component of the corresponding output would be the greatest for some of these inputs.

If a perturbation $\bm{\xi}^0 \in \mathbb{R}^N$ is small enough, then the input data $\bm{{\widetilde{x}}} = \bm{x}^0 + \bm{\xi}^0$ yields the output $\bm{{\widetilde{x}}}(T) \approx \bm{x}(T) + \bm{\xi}(T)$, where $\bm{\xi}$ is the solution to the corresponding input neural sensitivity problem~\eqref{eq:sp} with both $\omega(t)$ and $\beta(t)$ being identically zero (since neither the weights nor the biases are being perturbed here). 
Thus,~\eqref{eq:sp} becomes a homogeneous linear differential equation and the value of $\bm{\xi}(T)$ depends linearly on $\bm{\xi}^0$. In particular, there is a matrix $P \in \mathbb{R}^{N\times N}$ such that $\bm{\xi}(T) = P\bm{\xi}^0$.

Assume that $\bm{x}^0$ is inferred to belong to class $j$. For the sake of brevity, let $\bm{z} = \bm{x}(T)$ and $\bm{\zeta}=\bm{\xi}(T)$. The $i^{\text{th}}$ component of $\bm{{\widetilde{x}}}(T)$ will exceed the $j^{\text{th}}$ component whenever $\kappa ({z}_i + {\zeta}_i) \ge ({z}_j + {\zeta}_j)$, where the positive constant $\kappa \le 1$ determines a safety margin to compensate for a possible approximation error. Given this inequality, we seek $\bm{\xi}^0$ such that $P_{i,j} \bm{\xi}^0 = \bm{\zeta}_{i,j}$, where $P_{i,j}\in \mathbb{R}^{2\times N}$ is the submatrix of $P$ containing rows $i$ and $j$ of $P$ and $\bm{\zeta}_{i,j} = (\zeta_i, \zeta_j) \in\mathbb{R}^2$ consists of the components $i$ and $j$ of $\bm{\zeta}$. 
In the limiting case when $\kappa({z}_i + {\zeta}_i) = {z}_j + {\zeta}_j$, the least squares method yields that the least norm solution of the equation $P_{i,j} \bm{\xi}^0 = \bm{\zeta}_{i,j}$ is given by
\begin{equation}
	\begin{aligned}    
		\bm{\xi}^0 = P_{i,j}^{\mathsf{T}} (P_{i,j} P_{i,j}^{\mathsf{T}})^{-1} \bm{\zeta}_{i,j}, \quad
		\text{where} \quad {\zeta}_i & = \frac{\begin{pmatrix}1 & \kappa \end{pmatrix} (P_{i,j} P_{i,j}^{\mathsf{T}})^{-1} \begin{pmatrix}0 \\ {z}_j - \kappa {z}_i\end{pmatrix}}{\begin{pmatrix}1&\kappa\end{pmatrix} (P_{i,j} P_{i,j}^{\mathsf{T}})^{-1} \begin{pmatrix}1 \\ \kappa\end{pmatrix}}
		\\ \text{and}\quad {\zeta}_j &= \kappa {\zeta}_i + \kappa {z}_i - {z}_j.
	\end{aligned}
	\label{eq:xi0-attack}
\end{equation}
A candidate $\bm{x}^1$ for the perturbed input is then obtained by taking a short step (of predetermined length) from $\bm{x}^0$ in the direction of $\bm{\xi}^0$ given in~\eqref{eq:xi0-attack}. Should it happen that $\bm{x}^1$ ends up outside of the region of valid input values, then it is to be replaced by a nearest point within the valid region. If the classification label of $\bm{x}^1$ differs from $j$, then we have found an adversarial attack. Otherwise, we repeat the process by solving~\eqref{eq:xDE} with $\bm{x}^1$ as a starting point, then finding a new matrix $P$ for the solution of~\eqref{eq:sp} and a new direction of perturbation $\bm{\xi}^1$ as in~\eqref{eq:xi0-attack}, thereby generating another adversarial candidate $\bm{x}^2$. Additional iterations can be done as needed. A detailed algorithm to produce adversarial attacks by this method can be found in Appendix~\ref{app:adversarial}.

Our method exhibits a considerably higher success rate of adversarial attacks than state-of-the-art iterative attack methods, for most of the SGD-trained models, which can be seen in Figure~\ref{fig:advComp}~(left).~This comparison was done using the discrete ResNet-analogs of sNODEs for classification of downsampled MNIST, trained by SGD. We compare to the projected gradient descent (PGD) attack~\citep{madry2018towards}, which was most successful for our models, and in Appendix~\ref{app:adversarial} we include comparisons to FGSM~\citep{goodfellow2014explaining} and the method proposed by~\citet{carlini2017towards}. 
In all experiments we use the same step-length, $0.06$, and number of iterations, $15$.
Surprisingly, while the proposed method works better for the SGD-trained models, this does not apply to models trained with NCG, see Figure~\ref{fig:advComp} (middle and right). For examples of adversarial images, we refer to Appendix~\ref{app:adversarial}.

\textbf{NCG training and Sobolev gradient descent improve robustness against adversarial attacks.}
As can be seen in Figure~\ref{fig:advComp}, the NCG and Sobolev NCG sNODEs are significantly less susceptible to adversarial attacks compared to the SGD-trained networks. This goes for all the tested combinations of activation and loss functional. In general, including the Sobolev gradient in the NCG training further reduces the sensitivity to adversarial attacks, but not for all settings.
It has previously been demonstrated how NODEs compare favorably to conventional networks in terms of robustness to adversarial attacks~\citep{yan2019robustness,carrara2019robustness}, and our experiments confirm that this also apply to the standalone NODEs trained with NCG.

\begin{figure}[t]
	\begin{centering}
		\includegraphics[width=0.328\linewidth,trim=2mm 0 4mm 0, clip]{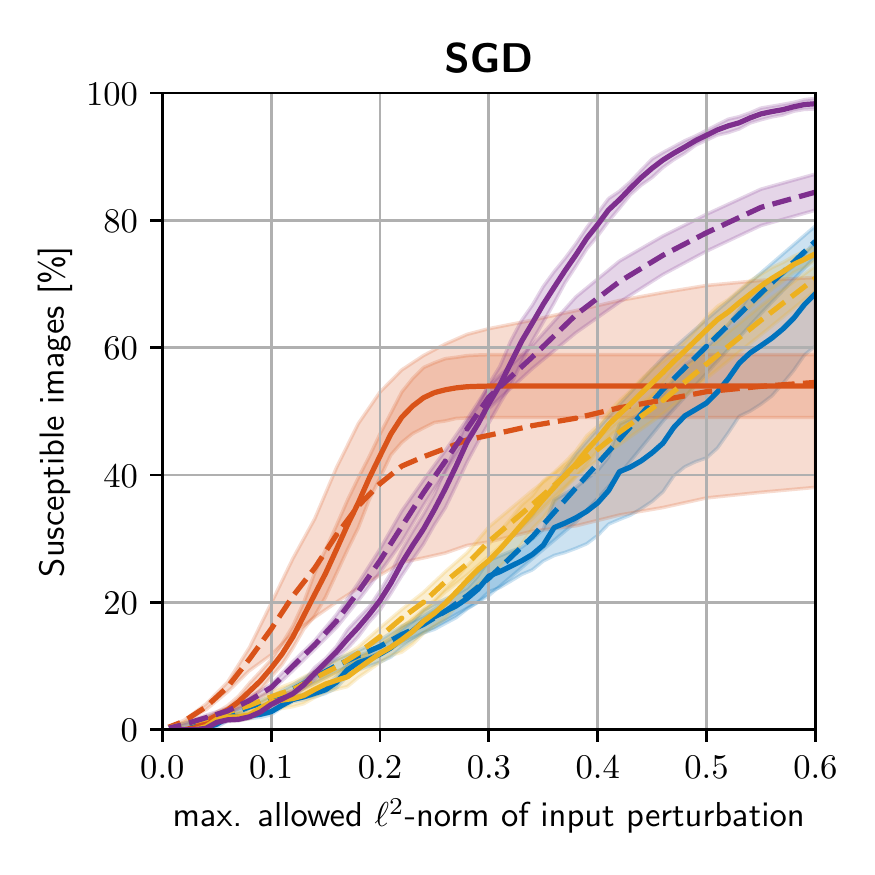} 
		\includegraphics[width=0.328\linewidth,trim=2mm 0 4mm 0, clip]{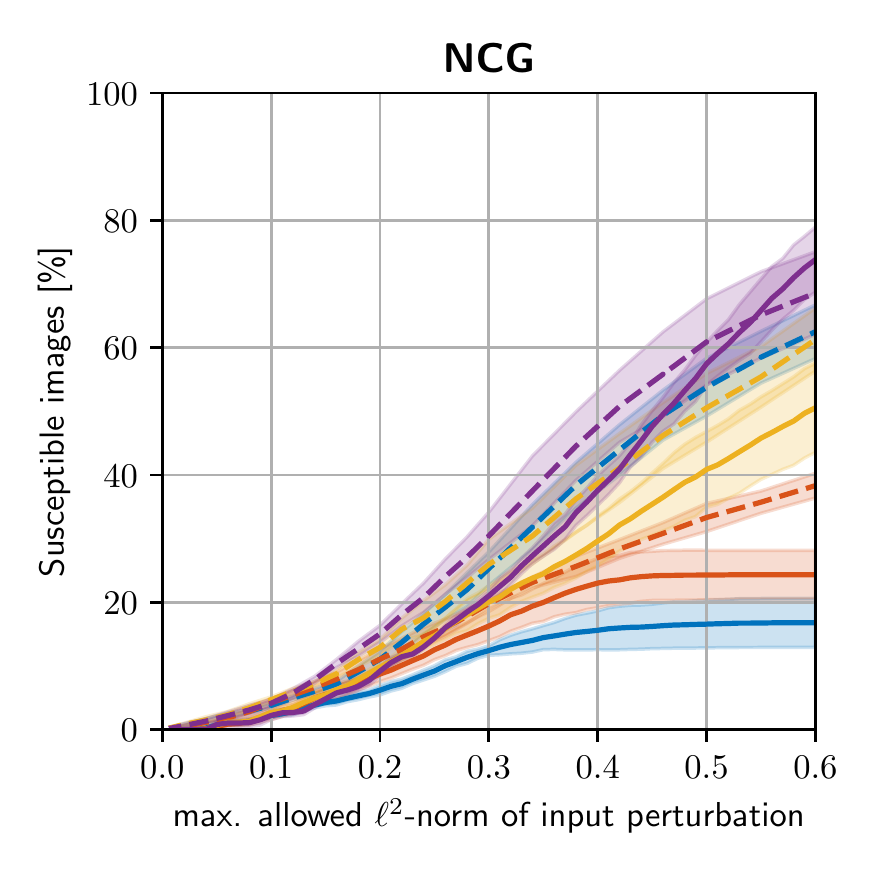} 
		\includegraphics[width=0.328\linewidth,trim=2mm 0 4mm 0, clip]{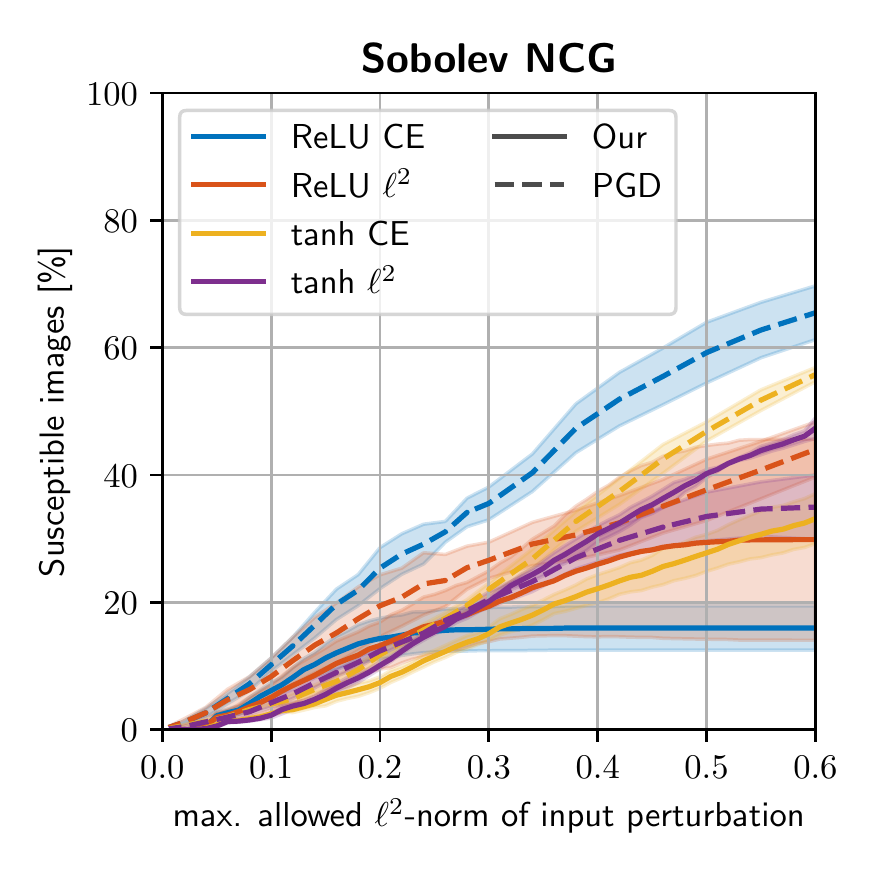} 
		\caption{Percentage of correctly classified images of test data set (MNIST downsampled to 14×14) that are susceptible to adversarial attacks. 
			On average, an input perturbation whose $\ell^2$-norm exceeds the threshold value of around $0.5$ is (subjectively) perceptible by the naked eye. The graphs show mean value and standard deviation over 5 test-runs for each training regime, separated between SGD (left), NCG (middle), and NCG with Sobolev gradient (right). The continuous line is our method, and dashed line is PGD~\citep{madry2018towards}. The legend (right) applies to all plots.}
		\label{fig:advComp}
	\end{centering}
	\vspace{-5mm}
\end{figure}

\section{Conclusion}
Novel sNODEs and a novel NCG are proposed and they compare well to SGD with a state-of-the-art optimizer, in the tested combinations of network and training setup (different activation function, loss function, and regularization). This is very promising, as this new setting has a great potential to further expand mathematical analysis of deep learning, from the previous embedded NODE blocks to an end-to-end network. In this work, we have focused on using this approach to perform sensitivity analysis of trained networks, and for generation of adversarial attacks. The experiments show that the solution to the neural sensitivity problem corresponds well to measured sensitivity and can be generalized to ResNet-models. Furthermore, the neural sensitivity problem can be used to generate adversarial examples. For SGD-trained models, these are more effective than previous methods at attacking networks trained on MNIST.
Another promising direction is to use Sobolev NCG for increased smoothness of the trained weights, and we have seen how this can have positive impact on robustness against adversarial attacks.

Although we can successfully optimize large sNODEs\footnote{The SGD network corresponding to the Euler discretized sNODE, with 150 layers and 196 neurons per layer, uses $\sim$5.8M trainable weights.}, with results on par or better than conventional SGD, there is a fundamental limitation in only considering fully connected layers. It is difficult in terms of memory and training time to scale the training to larger images. This is also the reason why we have only included results on the 14×14 pixels downsampled MNIST dataset, and with results that fall short of the state-of-the-art.

For future work, it is central to extend sNODEs beyond fully connected layers, e.g. by means of convolutional layers, and to allow for GPU acceleration. This will allow to fully exploit sNODEs on, e.g., large-scale computer vision problems. Other interesting research directions include sensitivity analysis of weight perturbations, and for uncertainty estimation. We can also see many promising possibilities in incorporating the adversarial sensitivity in the training, so as to further improve robustness against attacks. We believe that Standalone NODEs optimized by Sobolev NCG has the potential to bridge many gaps between deep learning and mathematics, promoting better theoretical foundation in future deep learning problems.

\section*{Acknowledgments}
This work was supported by the LiU Cancer network at Link\"{o}ping University, the research environment ELLIIT, and Link\"{o}ping University Center for Industrial Information Technology (CENIIT).

\newpage
\appendix
\section*{\huge{Appendices}}

\section{Mathematical background and proofs}\label{app:proofs}
In this section, we first give notations in~\ref{app:notation} then detailed proofs for the derivation of the neural sensitivity problem, the adjoint problem and the Sobolev descent direction in~\ref{app:sens}, \ref{app:adj} and \ref{app:sob} respectively.  

\subsection{Notations}\label{app:notation} 
The norms considered in this paper are denoted as follows:
\begin{itemize}
	\item $\| \bm{z} \|_{\ell^2}$ or simply $|\bm{z}|$ denotes the \emph{Euclidean norm} of the vector $\bm{z} \in \mathbb{R}^N$, i.e.,
	      \[ \| \bm{z} \|_{\ell^2} = \Bigl( \sum_{n=1}^N |z_n|^2 \Bigr)^{1/2}. \]
	\item $\| A \|_{F}$ denotes the \emph{Frobenius norm} of the matrix $A \in \mathbb{R}^{N\times N}$,
	      i.e.,
	      \[ \| A \|_{F} = \Bigl( \sum_{n,m=1}^N |A_{n,m}|^2 \Bigr)^{1/2}. \]
	\item $\| f \|_{L^2(I)}$ or simply $\| f \|_{L^2}$ is the \emph{$L^2$-norm} of a real-valued function $f: I \to \mathbb{R}$, i.e.,
	      \[ \| f \|_{L^2(I)} = \Bigl(\int_I |f(t)|^2 \,dt\Bigr)^{1/2}. \]
	      For a vector-valued mapping $b: I\to \mathbb{R}^N$ and a matrix-valued mapping $W: I\to \mathbb{R}^{N\times N}$, we define
	      \[
	      	\| b \|_{L^2} = \Bigl(\int_I \|b(t)\|_{\ell^2}^2 \,dt\Bigr)^{1/2}\quad \text{and} \quad
	      	\| W \|_{L^2} = \Bigl(\int_I \|W(t)\|_{F}^2 \,dt\Bigr)^{1/2}.
	      \]
	\item $\| f \|_{W^{1,2}(I)}$ is the first-order \emph{Sobolev norm} of a real-valued, vector-valued or matrix-valued mapping $f$, i.e.,
	      \[ \| f \|_{W^{1,2}(I)} = \bigl(\|f\|_{L^2(I)}^2 + \| f'\|_{L^2(I)}^2\bigr)^{1/2},\]
	      where $f'$ is the distributional derivative of $f$.
\end{itemize}
Observe that all the listed norms satisfy the parallelogram law and hence can be expressed using an inner product, namely $\| a \|^2 = \langle a, a \rangle$ provided that $a$ lies in the corresponding vector space. In particular, for functions $f, g: I \to \mathbb{R}$ we have
\begin{align*}
	\langle f, g \rangle_{L^2(I)} = \int_I f(t) g(t)\,dt \quad\text{and}\quad             
	\langle f, g \rangle_{W^{1,2}(I)} = \int_I \Bigl( f(t) g(t) + f'(t)g'(t) \Bigr) \,dt. 
\end{align*}
In case $f$ and $g$ are vector or matrix-valued, then $f(t)g(t)$ and $f'(t) g'(t)$ in the integrals above are to be replaced by the $\ell^2$ or Frobenius inner product of $f(t)$ and $g(t)$, and $f'(t)$ and $g'(t)$, respectively.
For the sake of simplicity, both $\langle \cdot , \cdot \rangle_{\ell^2}$ and $\langle \cdot , \cdot \rangle_{F}$ are denoted by $\langle \cdot , \cdot \rangle$ in Section~\ref{sec:ncg} of the paper, Section~\ref{app:adj} and Section~\ref{app:learn} of the appendix. 
\subsection{Neural sensitivity problem}\label{app:sens}
\begin{proof}[Proof of Lemma~\ref{lemmaSens}]
	Assuming that the element $\bm{x}(t)$, which is a solution to~(\ref{eq:xDE}), is perturbed by $\varepsilon \bm{\xi}(t)$ when $W(t)$ and $b(t)$ are perturbed by $\varepsilon \omega(t)$ and $\varepsilon \beta(t)$ respectively, with $\varepsilon>0$ being a small number, the corresponding perturbed problem is then given by
	\begin{equation}\label{eq:perturb}
		\left\{\begin{aligned}
		( \bm{x}(t) + \varepsilon \bm{\xi}(t) )' &= \sigma((W(t)+\varepsilon \omega(t)) (\bm{x}(t) + \varepsilon \bm{\xi}(t)) + (b(t)+\varepsilon \beta(t))),& t\in I,\\
		\bm{x}(0) + \varepsilon \bm{\xi}(0) & =  \bm{x}^{0} + \varepsilon \bm{\xi}^{0}.
		\end{aligned}\right.
	\end{equation}
	Subtracting problem~(\ref{eq:xDE}) from~\eqref{eq:perturb}, we obtain
	\begin{equation*}
		\left\{\begin{aligned}
		\varepsilon \bm{\xi}'(t) &= \sigma(W(t) \bm{x}(t)+ b(t) + \varepsilon (\omega(t) \bm{x}(t) + \beta(t) + W(t) \bm{\xi}(t) + \varepsilon \omega(t)))\\[0.1ex]
		&\qquad - \sigma(W(t)\bm{x}(t)+b(t)),& t\in I,\\[0.2ex]
		\varepsilon \bm{\xi}(0) & =  \varepsilon \bm{\xi}^{0} .
		\end{aligned}\right.
	\end{equation*}
	The neural sensitivity problem is then obtained by letting $\varepsilon \to 0$, and hence
	\begin{equation*}
		\left\{
		\begin{aligned}
			\bm{\xi}'(t) & = \sigma'\bigl(W(t) \bm{x}(t) + b(t)\bigr) \odot \bigl( \omega(t) \bm{x}(t) + \beta(t) + W(t) \bm{\xi}(t) \bigr), & t\in I, \\
			\bm{\xi}(0) & = \bm{\xi}^{0},
			\\ 
		\end{aligned}
		\right.  
	\end{equation*}
	where $\sigma'$ is the derivative of $\sigma$ applied componentwise and $\odot$ denotes the Hadamard (i.e., elementwise) product.
\end{proof}

\subsection{Adjoint Problem}\label{app:adj}
\begin{proof}[Proof of Lemma~\ref{lemmaAdj}]
	Let $\omega(t)$ and $\beta(t)$ be the directions of the perturbations of $W(t)$ and $b(t)$, respectively.
	The directional derivative of $\mathcal{E}(W,b)$ is given by
	\begin{align*}
		\mathcal{E}'_{(\omega,\beta)} & (W,b) = \lim_{\varepsilon \to 0} \frac{1}{\varepsilon}\bigl(\mathcal{E}(W+\varepsilon \omega, b+\varepsilon \beta) - \mathcal{E}(W,b)\bigr)                                                                                                                        \\
		                              & = \frac{1}{K} \sum_{k=1}^K \lim_{\varepsilon \to 0} \frac{1}{\varepsilon} \Biggl(\frac{\mu_1}{2}  \bigl(\| \bm{x}_k(T)+\varepsilon \bm{\xi}_k(T) - \bm{y}_k \|_{\ell^2}^2-\| \bm{x}_k(T) - \bm{y}_k \|_{\ell^2}^2\bigr)                                            \\
		                              & \qquad \qquad + \mu_2  \bigl(H\bigl(\bm{y}_k, \smax (\bm{x}_k(T)+\varepsilon \bm{\xi}_k(T))\bigr)-H\bigl(\bm{y}_k, \smax (\bm{x}_k(T))\bigr)\bigr)                                                                                                                 \\
		                              & \qquad \qquad + \frac{\mu_3}{2} \bigl(\| \bm{x}_k(T)+\varepsilon \bm{\xi}_k(T)  \|_{\ell^2}^2 - \| \bm{x}_k(T)  \|_{\ell^2}^2\bigr)\biggr)                                                                                                                         \\
		                              & \quad + \frac{\mu_4}{2} \int_0^T \lim_{\varepsilon \to 0} \frac{1}{\varepsilon}\bigl(  \|W+\varepsilon \omega\|^2_{F}-\|W\|^2_{F} + \|b+\varepsilon \beta\|^2_{\ell^2}- \|b\|^2_{\ell^2}\bigr) dt                                                                  \\
		                              & \quad + \frac{1}{K} \sum_{k=1}^K \int_{0}^{T} \bigl\langle \bm{\lambda}_{k}, \lim_{\varepsilon \to 0} \frac{1}{\varepsilon} \bigl( \sigma(W \bm{x}_k+b +\varepsilon^2 \omega\bm{\xi}_k+\varepsilon (W \bm{\xi}_k +\omega \bm{x}_k+\beta))                          \\
		                              & \qquad \qquad -\sigma(W \bm{x}_k +b) - ((\bm{x}_k+\varepsilon \bm{\xi}_k)'-\bm{x}_k') \bigr) \bigr\rangle\,dt                                                                                                                                                      \\ 
		\quad                         & =  \frac{1}{K} \sum_{k=1}^K \bigl\langle \bm{\xi}_k(T), \mu_1 (\bm{x}_k(T)-\bm{y}_k)+ \mu_2 (\smax (\bm{x}_k(T))-\bm{y}_k) + \mu_3 \bm{x}_k(T)\bigr\rangle                                                                                                         \\
		                              & \quad + \mu_4 \int_0^T (\bigl\langle \omega, W \bigr\rangle +\bigl\langle \beta, b \bigr\rangle)\,dt                                                                                                                                                               \\
		                              & \quad + \frac{1}{K} \sum_{k=1}^K \int_{0}^{T} \bigl\langle \bm{\lambda}_{k},\sigma'(W \bm{x}_k +b) \odot (W \bm{\xi}_k+\omega \bm{x}_k+\beta)\bigr\rangle\,dt - \frac{1}{K} \sum_{k=1}^K \int_{0}^{T} \bigl\langle \bm{\lambda}_{k}, \bm{\xi}_k'\bigr\rangle \,dt. 
	\end{align*}
	Integration by parts in the last integral, using $\bm{\xi}_k(0)=0$ and rearranging terms, give 
	\begin{align*}
		\mathcal{E}'_{(\omega,\beta)}(W,b) & =\frac{1}{K} \sum_{k=1}^K \bigl\langle \bm{\xi}_k(T), \mu_1 (\bm{x}_k(T)-\bm{y}_k)+ \mu_2 (\smax (\bm{x}_k(T))-\bm{y}_k) + \mu_3 \bm{x}_k(T)-\bm{\lambda}_k(T)\bigr\rangle \\ & \quad +  \int_{0}^{T} \bigl\langle \mu_4 W + \frac{1}{K} \sum_{k=1}^K \bm{x}_k^\T (\sigma'(W \bm{x}_k +b) \odot \bm{\lambda}_{k}),\omega \bigr\rangle\,dt \\ & \quad + \int_{0}^{T} \bigl\langle \mu_4 b + \frac{1}{K} \sum_{k=1}^K (\sigma'(W \bm{x}_k +b) \odot \bm{\lambda}_{k}),\beta \bigr\rangle\,dt \\ & \quad +\frac{1}{K} \sum_{k=1}^K \int_{0}^{T} \bigl\langle \bm{\lambda}_{k}' - W^\T \bigl( \sigma'\bigl( W \bm{x}_k + b \bigr) \odot \bm{\lambda}_k \bigr), \bm{\xi}_k\bigr\rangle \,dt.
	\end{align*}
	If the Lagrange multiplier $\bm{\lambda}_k$ is a solution to the \textit{adjoint problem} given by
	\[
		\left\{
		\begin{aligned}
			\bm{\lambda}_k'(t) & = - W(t)^\T \Bigl( \sigma'\bigl( W(t) \bm{x}_k(t) + b(t) \bigr) \odot \bm{\lambda}_k(t)\Bigr), \\
			\bm{\lambda}_k(T)  & = \mu_1(\bm{x}_k(T) - \bm{y}_k)+ \mu_2 (\smax(\bm{x}_k(T))-\bm{y}_k)+\mu_3 \bm{x}_k(T),        
		\end{aligned}
		\right.
	\]
	then we obtain that
	\begin{align*}
		\mathcal{E}'_{(\omega,\beta)}(W,b)  & =                                                                                                             
		\int_{0}^{T} \bigl\langle \mu_4 W + \frac{1}{K} \sum_{k=1}^K \bm{x}_k^\T (\sigma'(W \bm{x}_k +b) \odot \bm{\lambda}_{k}),\omega \bigr\rangle\,dt \\ & \quad +  \int_{0}^{T} \bigl\langle \mu_4 b + \frac{1}{K} \sum_{k=1}^K (\sigma'(W \bm{x}_k +b) \odot \bm{\lambda}_{k}),\beta \bigr\rangle\,dt.
	\end{align*}
	Hence, 
	\begin{align*}
		\mathcal{E}'_{\omega} & = \mu_4 W + \frac{1}{K} \sum_{k=1}^K \bm{x}_k^\T (\sigma'(W \bm{x}_k +b) \odot \bm{\lambda}_{k}) \quad\text{and} \\
		\mathcal{E}'_{\beta}  & = \mu_4 b + \frac{1}{K} \sum_{k=1}^K (\sigma'(W \bm{x}_k +b) \odot \bm{\lambda}_{k}). \qedhere                   
	\end{align*}
\end{proof}

\subsection{Sobolev descent direction}\label{app:sob}
\begin{proof}[Proof of Corollary~\ref{cor:SobDer}]
	Since every element of the dual space of $L^2(0, T)$ defines a continuous functional on the Sobolev class $W^{1,2}(0,T)$, the Riesz representation theorem for Hilbert spaces implies that there is a Sobolev function that represents the functional with respect to the inner product of $W^{1,2}(0, T)$. 
	We need to show that the transformation $\mathcal{S}$ as described in~(\ref{eq:S-xform}) yields such a representative, i.e.,
	\[ 
		\bigl< u, \phi \bigr>_{L^2(0,T)} = \bigl< \mathcal{S}u, \phi \bigr>_{W^{1,2}(0,T)}
	\]
	holds true for every $u \in L^2(0,T) \cong \bigl(L^2(0, T)\bigr)^*$ and every $\phi \in W^{1,2}(0, T)$. In other words, we want to show that
	\[
		\int_0^T u(t) \phi(t)\,dt = \int_0^T \bigl(\mathcal{S}u(t) \phi(t) + (\mathcal{S}u)'(t) \phi'(t)\bigr)\,dt.
	\]
	If $u \in \bigl(L^2(0,T)\bigr)^*$ is continuous, then a direct calculation shows that $\mathcal{S} u \in \mathcal{C}^2(0,T) \cap \mathcal{C}^{0,1}[0,T]$ is a solution to the boundary value problem
	\[
		\left\{
		\begin{aligned}
			(\mathcal{S} u )''(t) - \mathcal{S} u (t) & = -u(t), \quad t \in (0, T), \\
			(\mathcal{S} u )'(0)                      & = 0,                         \\
			(\mathcal{S} u )'(T)                      & = 0.                         
		\end{aligned}
		\right.
	\]
	Thus, we obtain by integration by parts that
	\begin{align*}
		\int_0^T & \bigl(\mathcal{S} u(t) \phi(t) + (\mathcal{S} u)'(t) \phi'(t)\bigr)\,dt                                                                \\
		         & = (\mathcal{S} u)'(T) \phi(T) - (\mathcal{S} u)'(0) \phi(0) - \int_0^T \bigl((\mathcal{S} u)''(t) - \mathcal{S} u(t)\bigr) \phi(t)\,dt \\
		         & = \int_0^T u(t) \phi(t)\,dt                                                                                                            
	\end{align*}
	for every $\phi \in W^{1,2}(0,T)$. Since continuous functions are dense in $L^2(0,T)$, we may apply density arguments to verify that $\bigl< u, \phi \bigr>_{L^2(0,T]} = \bigl< \mathcal{S}u, \phi \bigr>_{W^{1,2}(0,T)}$ is satisfied for a general, possibly discontinuous, $u \in \bigl(L^2(0, T)\bigr)^* \cong L^2(0,T)$.
\end{proof}
Formul\ae{}~(\ref{eq:E'W}) and~(\ref{eq:E'b}) give a representative of the Fréchet derivative of the functional $\mathcal{E}$ with respect to the $L^2$-inner product. We have shown above that the transformation $\mathcal{S}$ provides the corresponding representative with respect to the $W^{1,2}$-inner product.
\section{Nonlinear Conjugate Gradient method}\label{app:ncg}
In this section, we present the algorithm for the nonlinear conjugate gradient iterations for finding weights and biases, along with details in~\ref{app:learn} on how to compute an optimal learning rate via the neural sensitivity problem. We also show in~\ref{app:smooth} an example illustrating the smoothening effects of the Sobolev gradient descent direction.
\begin{algorithm}[H]
	\caption{Nonlinear Conjugate Gradient (NCG)}
	\label{alg:ncg}
	\begin{algorithmic}[1]
		{
			\renewcommand{\algorithmicrequire}{\textbf{Input:}}
			\renewcommand{\algorithmicensure}{\textbf{Output:}}
			\REQUIRE Batch of K randomly selected images from the training set with corresponding target classes
			\REQUIRE An activation function $\sigma$ and its derivative
			\REQUIRE Regularization coefficients $\mu_i$, where $i=1, \ldots, 4$
			\REQUIRE Initial weights and biases $W^0$ and $b^0$
			\FOR {$j=0,1,2\ldots$} 
			\STATE Solve the direct problem~(\ref{eq:xDE})
			\STATE Solve the adjoint problem~(\ref{eq:aDE2})
			\STATE Compute the gradient from~(\ref{eq:E'W}) and~(\ref{eq:E'b})
			\STATE [\emph{For Sobolev descent only:} Apply $\mathcal{S}$ from~(\ref{eq:S-xform}) componentwise to the gradient.]
			\IF{$j=0$}
			\STATE $\gamma^j = 0$
			\ELSE
			\STATE Compute Fletcher-Reeves conjugate gradient coefficient 
			\begin{equation*}
				\gamma^j = \frac{\| \mathcal{E}'_{\omega}(W^j,b^j) \|_{L^2}^2+\| \mathcal{E}'_{\beta}(W^j,b^j) \|_{L^2}^2}{\| \mathcal{E}'_{\omega}(W^{j-1},b^{j-1}) \|_{L^2}^2+\| \mathcal{E}'_{\beta}(W^{j-1},b^{j-1}) \|_{L^2}^2}   
			\end{equation*}
			\ENDIF
			\STATE Compute descent directions 
			\begin{align*}
				dW^{j} = -\mathcal{E}'_{\omega}(W^j,b^j) + \gamma^{j} \,dW^{j-1} \quad\text{and}\quad db^{j} =-\mathcal{E}'_{\beta}(W^j,b^j)  + \gamma^{j}\,db^{j-1} 
			\end{align*}
			\STATE Solve the neural sensitivity problem~(\ref{eq:sp}) with $\bm{\xi}_{k}^{j}(0)=0$, $\omega(t) = dW^{j}(t)$ and $\beta(t) = db^{j}(t)$
			\STATE Perform a line search to find the optimal learning rate \begin{equation*}
			\eta\,^{j}=\mathop{\arg\min}\limits_{\eta>0}\, E(W^{j}+\eta\, dW^{j}, b^{j}+\eta\, db^{j})
			\end{equation*}
			\STATE Update \begin{align*}
			\qquad	  W^{j+1} = W^{j} + \eta\,^{j}\,dW^{j} \quad\text{and}\quad b^{j+1} = b^{j} + \eta\,^{j}\,db^{j}
			\end{align*}
			\ENDFOR
		}
	\end{algorithmic} 
\end{algorithm}

\subsection{Computation of the NCG learning rate}\label{app:learn}
The optimal learning rate of the nonlinear conjugate gradient method is obtained from
\begin{equation*}
	\eta\,^{j}=\mathop{\arg\min}\limits_{\eta>0}\, E(W^{j}+\eta\, dW^{j}, b^{j}+\eta\, db^{j}).
\end{equation*}
If $\bm{\xi}^j_{k}$ is a solution to the neural sensitivity problem~(\ref{eq:sp}) with $\bm{\xi}_{k}^{j}(0)=0$, $\omega(t) = dW^{j}(t)$ and $\beta(t) = db^{j}(t)$, then
\begin{align*}
	E( & W^j+\eta\,dW^j, b^j+\eta\, db^j)                                                                                                  \\
	   & \approx  \frac{1}{K} \sum_{k=1}^K \Biggl(\frac{\mu_1}{2} \bigl\| \bm{x}_k(T)+\eta\, \bm{\xi}^j_k(T) - \bm{y}_k \bigr\|_{\ell^2}^2 
	+ \mu_2  H\bigl(\bm{y}_k, \smax (\bm{x}_k(T)+\eta\, \bm{\xi}^j_k(T))\bigr) \\
	   & \qquad \qquad\quad+ \frac{\mu_3}{2} \bigl\| \bm{x}_k(T)+\eta\, \bm{\xi}^j_k(T) \bigr\|_{\ell^2}^2 \Biggr)                         
	+ \frac{\mu_4}{2} \int_0^T \bigl(  \bigl\|W^j+\eta\, dW^j\bigr\|^2_{F} + \bigl\|b^j+\eta\, db^j\bigr\|^2_{\ell^2}\bigr) dt.
\end{align*}
Differentiating the right hand side with respect to $\eta$ yields
\begin{align*}
	  & \eta \biggr(  \frac{\mu_1 + \mu_3}{K} \sum_{k=1}^K \bigl\| \bm{\xi}^j_k(T) \bigr\|_{\ell^2}^2 + \mu_4 \bigl(\bigl\|dW^j\bigr\|_{L^2}^2 + \bigl\|db^j\bigr\|^2_{L^2}\bigr)\biggr) + \mu_2 \Bigl\langle \smax (\bm{x}_k(T)+\eta\, \bm{\xi}^j_k(T)),  \bm{\xi}^j_k(T)\Bigr\rangle \\
	  & \ \ + \frac{1}{K} \sum_{k=1}^K \bigl\langle \mu_1 (\bm{x}_k(T)-\bm{y}_k) - \mu_2 \bm{y}_k + \mu_3 \bm{x}_k(T), \bm{\xi}^j_k(T)\bigr\rangle + \mu_4 \int_{0}^{T} \bigl( \bigl\langle W^j, dW^j \bigr\rangle + \bigl\langle b^j, db^j \bigr\rangle \bigr)\,dt                    
\end{align*}
whose zero approximates the desired optimal learning rate $\eta^j$ for updating weights and biases. It is worth noting that this expression is linear in $\eta$ whenever $\mu_2 = 0$, i.e., whenever the cost functional does not contain cross-entropy loss.
\subsection{Smoothening effects of the Sobolev gradient descent direction}\label{app:smooth}
As can be seen in the example shown in Figure~\ref{fig:WBplot}, the steepest gradient descent direction with respect to the Sobolev $W^{1,2}$-inner product obtained by the transformation~(\ref{eq:S-xform}) results in smoother trained weights compared to weights obtained when considering the $L^2$-inner product for sNODEs and SGD for ResNets.
\begin{figure}[ht]
	\centering
	\includegraphics[width=\textwidth]{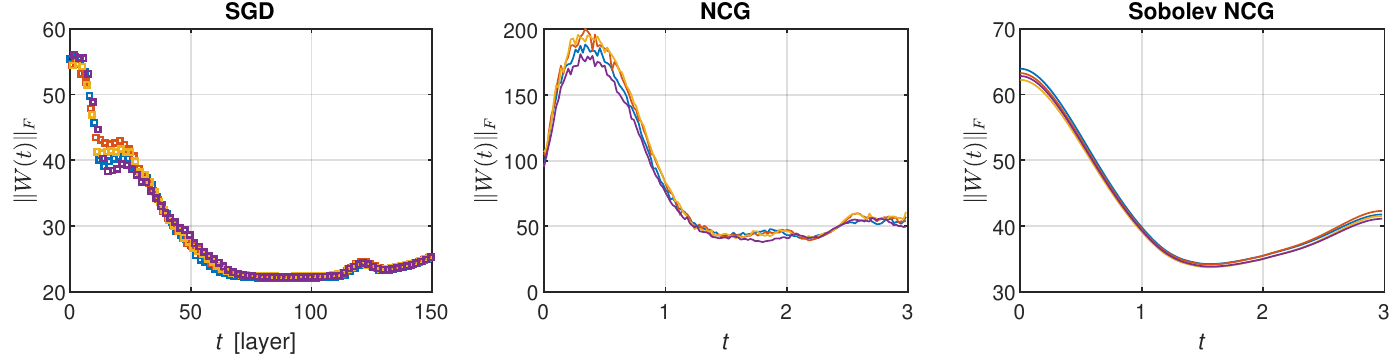}
	\caption{Comparison of the behavior of $\|W(t)\|_F$ obtained by training sNODEs for classification of downsampled (14×14) MNIST with $\ell^2$-loss and $\tanh$ as an activation function using different training methods. The plots show the norms of $W(t)$ for 4 testruns. Note the smoothness of the weights when Sobolev gradient descent was used during the training.}
	\label{fig:WBplot}
\end{figure}

\section{Evaluation of the neural sensitivity problem}\label{app:sensitivity}
In this section, we extend the evaluation of the neural sensitivity problem to the CE-loss setting: the estimated perturbation $\lVert \bm{\xi}(t) \rVert / \lVert \bm{x}(t) \rVert$ and the measured perturbation $\lVert \bm{\delta}(t) \rVert / \lVert \bm{x}(t) \rVert$ throughout the SGD-trained ResNets and the NCG-trained sNODEs with CE-loss are presented in Figures~\ref{fig:sp_xi_delta_ce_sgd} and ~\ref{fig:sp_xi_delta_ce_ncg}, respectively.  

\begin{figure}[H]
	\begin{centering}
		\begin{subfigure}[b]{0.43\textwidth}
			\centering
			\includegraphics[width=\textwidth]{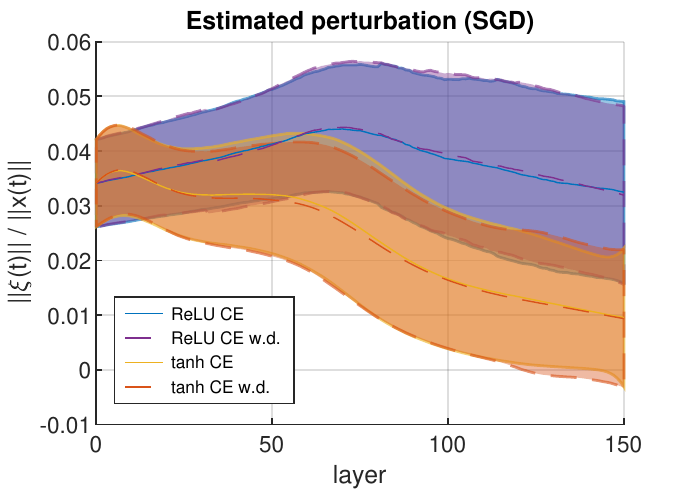}
		\end{subfigure}
		\begin{subfigure}[b]{0.43\textwidth}
			\centering
			\includegraphics[width=\textwidth]{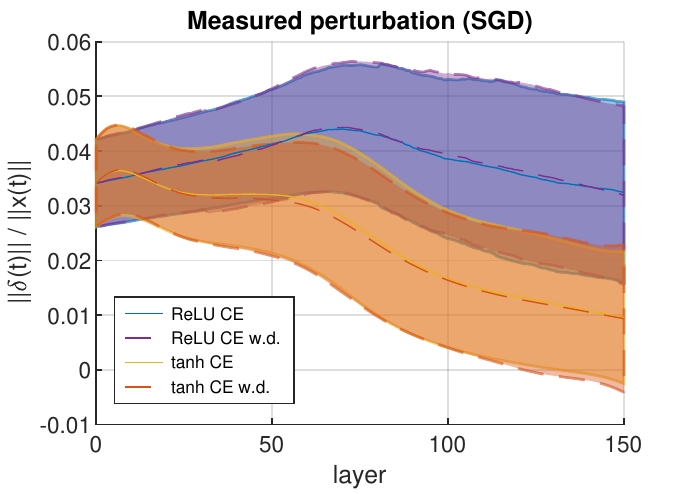}
		\end{subfigure}
		\caption{Estimated propagation of the perturbation (left panel) compared to the measured perturbation (right panel) when $\std(\bm{\xi}^0)=\std(\bm{\delta}^0)=0.01$ through different SGD-trained ResNets using CE-loss.}
		\label{fig:sp_xi_delta_ce_sgd}
	\end{centering}
\end{figure}

\begin{figure}[H]
	\begin{centering}
		\begin{subfigure}[b]{0.43\textwidth}
			\centering
			\includegraphics[width=\textwidth]{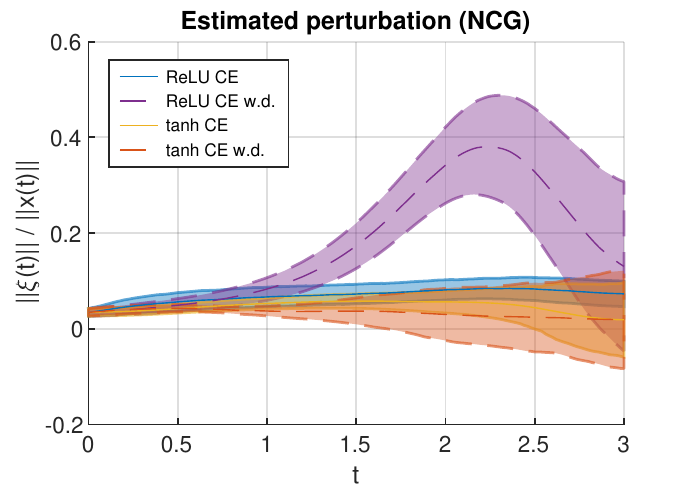}
		\end{subfigure}
		\begin{subfigure}[b]{0.43\textwidth}
			\centering
			\includegraphics[width=\textwidth]{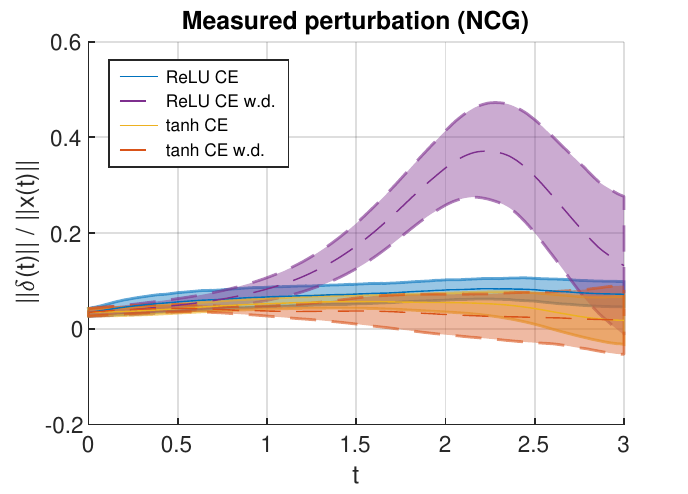}
		\end{subfigure}
		\caption{Estimated propagation of the perturbation (left panel) compared to the measured perturbation (right panel) when $\std(\bm{\xi}^0)=\std(\bm{\delta}^0)=0.01$ through different NCG-trained sNODEs using CE-loss. } 
		\label{fig:sp_xi_delta_ce_ncg}
	\end{centering}
\end{figure}

\vfill
We further illustrate how the general sensitivity problem provides a good approximate description of noise propagation by reporting relative estimation errors $\lVert \bm{\xi}(t)-\bm{\delta}(t) \rVert / \lVert \bm{\xi}(t) \rVert$ for CE-loss and $\ell^2$-loss in Figures~\ref{fig:sp_sgd_err_ce} and~\ref{fig:sp_sgd_err_app}, respectively. 

\vfill
\begin{figure}[H]
	\includegraphics[width=0.49\linewidth]{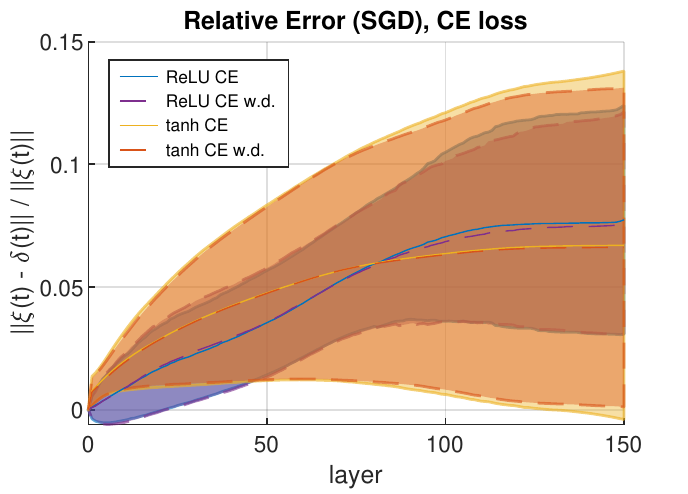} \hfill
	\includegraphics[width=0.49\linewidth]{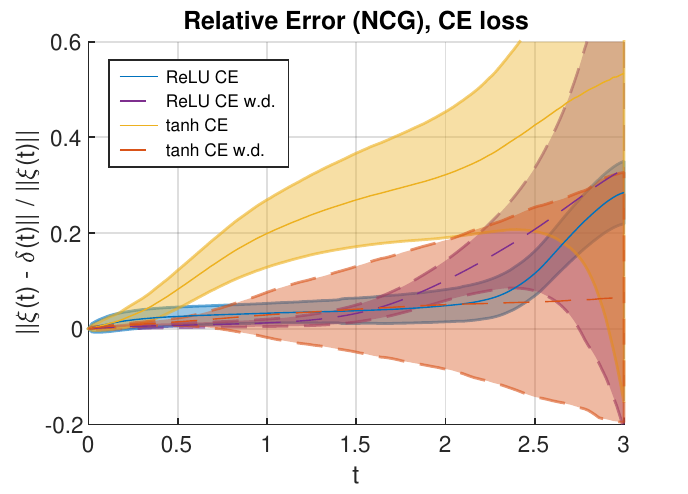}
	\caption{Relative error between estimated and measured perturbations for SGD-trained ResNets (left panel) and NCG-trained sNODEs (right panel) as $\std(\bm{\xi}^0)=\std(\bm{\delta}^0)=0.01$, where CE-loss was used during the training.}
	\label{fig:sp_sgd_err_ce}
\end{figure}

\vfill
All graphs in Figures~\ref{fig:sp_xi_delta_l2},~\ref{fig:sp_xi_delta_ce_sgd},~\ref{fig:sp_xi_delta_ce_ncg},~\ref{fig:sp_sgd_err_ce} and~\ref{fig:sp_sgd_err_app} show mean value and standard deviation over 5 test-runs for each training regime with the dashed line corresponding to settings where weight decay is considered. 

\vfill
\begin{figure}[H]
	\includegraphics[width=0.49\linewidth]{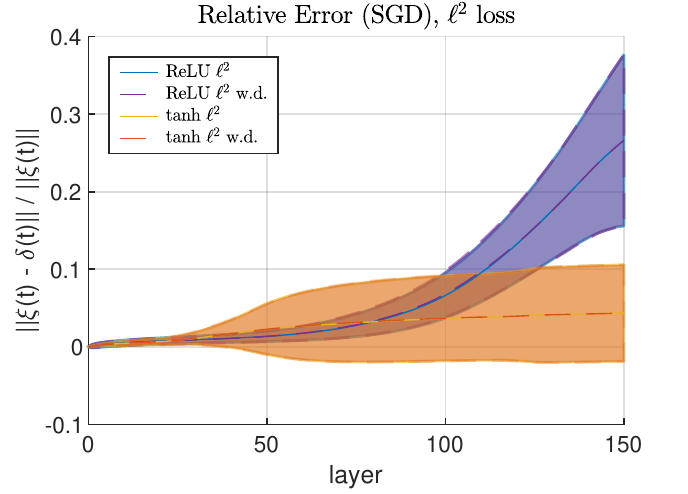} \hfill
	\includegraphics[width=0.49\linewidth]{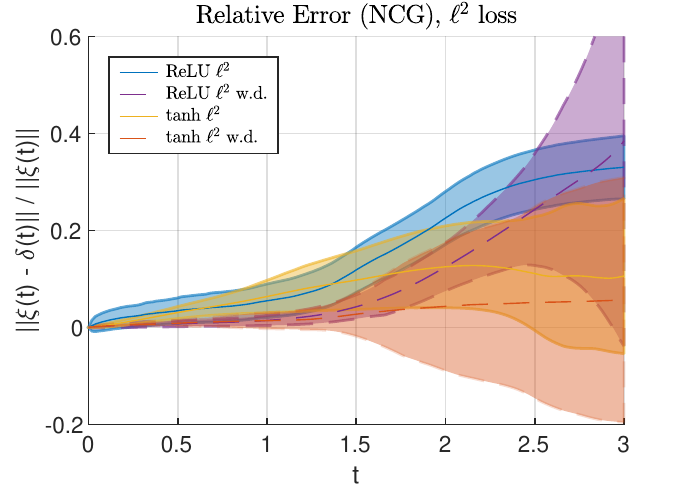}
	\caption{Relative error between estimated and measured perturbations for SGD-trained ResNets (left panel) and NCG-trained sNODEs (right panel) as $\std(\bm{\xi}^0)=\std(\bm{\delta}^0)=0.01$, where $\ell^2$-loss was used during the training, cf.~Figure~\ref{fig:sp_xi_delta_l2}}
	\label{fig:sp_sgd_err_app}
\end{figure}
We note from comparing Figure~\ref{fig:sp_xi_delta_l2} with Figures~\ref{fig:sp_xi_delta_ce_sgd}--~\ref{fig:sp_xi_delta_ce_ncg} that in the CE-setting, similarly to the $\ell^2$-loss setting, perturbations are more controlled when weight decay is applied for NCG-trained sNODEs with tanh activation functions whereas for SGD-trained ResNets, weight decay has no effect. We also report that considering CE-loss yields less sensitive SGD-trained ResNets with ReLU activation functions while it has no significant effect on the sensitivity of NCG-trained sNODEs.

\subsection*{Effect of Gaussian noise perturbations}
We display in Figure~\ref{fig:noisyImg} an illustration of the effects of Gaussian noise for several of the values of the standard deviation used to study robustness of different models and training settings, see Figure~\ref{fig:all_acc_nois}.  
\begin{figure}[h]
	\centering
	\includegraphics[width=\textwidth]{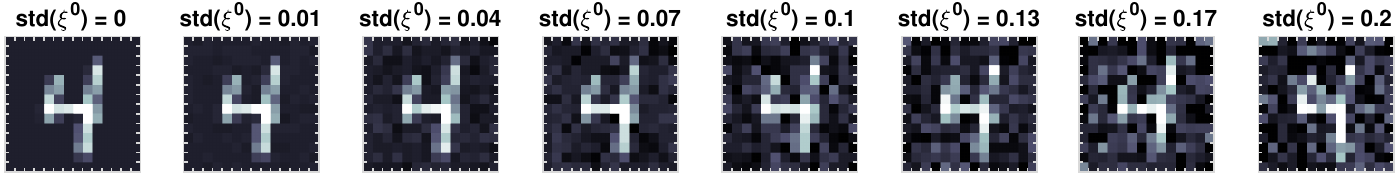}
	\caption{Illustration of effects of the Gaussian noise on an image from the downsampled (14×14) MNIST data set for several values of the standard deviation.}
	\label{fig:noisyImg}
\end{figure}

\section{Adversarial attacks and robustness}\label{app:adversarial}
In this section, we describe the main idea behind generating adversarial attacks via the neural sensitivity problem~(\ref{eq:sp}). First, we derive the mathematical tools for targeted and untargeted attacks and then we summarize these in form of pseudocode, see Algorithm~\ref{alg:advAttack} in Subsection~\ref{sec:attackSum}.

\subsection{Preliminaries}
In order to infer the class of a given input data point $\bm{x}^0 \in \mathbb{R}^N$, one solves the sNODE direct problem~(\ref{eq:xDE}) and determines $\arg\max$ of the output $\bm{x}(T)$. Provided that the input perturbation $\bm{\xi}^0 \in \mathbb{R}^N$ is sufficiently small, the perturbed input data $\bm{{\widetilde{x}}}^0 = \bm{x}^0 + \bm{\xi}^0$ yields the output $\bm{{\widetilde{x}}}(T) \approx \bm{x}(T) + \bm{\xi}(T)$, where $\bm{\xi}$ is the solution to the corresponding input neural sensitivity problem~(\ref{eq:sp}) with both $\omega(t)$ and $\beta(t)$ being identically zero, i.e.,
\begin{equation}
	\left\{
	\begin{aligned}
		\bm{\xi}'(t) & = \sigma'\bigl(W(t) \bm{x}(t) + b(t)\bigr) \odot W(t) \bm{\xi}(t), & t\in I, \\
		\bm{\xi}(0) & = \bm{\xi}^0.
	\end{aligned}
	\right.  
	\label{eq:spAdv}
\end{equation} 
This is a homogeneous linear differential equation and the value of $\bm{\xi}(T)$ depends linearly on $\bm{\xi}^0$. Thus, there is a matrix $P \in \mathbb{R}^{N\times N}$ such that $\bm{\xi}(T) = P\bm{\xi}^0$. In the next paragraph, we describe how to find this matrix $P$.

Assume that $\{\bm{e}_1, \bm{e}_2, \ldots, \bm{e}_N\}$ is the canonical orthonormal basis in $\mathbb{R}^N$. Then, the $n$\textsuperscript{th} column of $P$ is equal to $\bm{\xi}(T)\in \mathbb{R}^{N}$ obtained by solving \eqref{eq:spAdv} with the initial condition $\bm{\xi}(0) = \bm{e}_n$.
Instead of solving the sensitivity problem \eqref{eq:spAdv} for each of the basis vectors separately, one can compute the entire matrix $P$ at once by solving this ODE-system for the initial condition $\bm{\xi}(0)$ set equal to the  $N \times N$-identity matrix. Then, $P = \bm{\xi}(T) \in \mathbb{R}^{N\times N}$.

Note that the matrix $P$ can be interpreted as the Jacobian matrix for the derivative of the network output $\bm{x}(T)$ with respect to the network input $\bm{x}^0$. It is therefore possible to replace the usage of the sensitivity problem \eqref{eq:spAdv} by instead differentiating the network output with respect to the network input as in \citep{novak2018sensitivity}. However, the sensitivity problem would be preferable due to the computational cost, taking into account the number of nonlinear operations needed to find $P$. 
\subsection{Targeted attack}
\label{ssec:tgtAttacks}
Here, we assume that the neural network classifies the input $\bm{x}^0$ as class $j$. We search for a (small) perturbation $\bm{\xi}^0$ such that the perturbed input $\widetilde{\bm{x}}^0 := \bm{x}^0 + \bm{\xi}^0$ would be classified as class $i$, where $i \neq j$.

Since inference is based on the value of $\arg\max$ of the output layer, we need to focus on components of the output for the original data $\bm{x}(T)$ and for the perturbed data $\widetilde{\bm{x}}(T) \approx \bm{x}(T) + \bm{\xi}(T)$. Let us therefore define $\bm{z} = \bm{x}(T)$ and $\bm{\zeta}=\bm{\xi}(T)$.

The $i$\textsuperscript{th} component of $\bm{{\widetilde{x}}}(T) \approx \bm{z} + \bm{\zeta}$ will exceed its $j$\textsuperscript{th} component whenever $\kappa ({z}_i + {\zeta}_i) \ge ({z}_j + {\zeta}_j)$, where the constant $\kappa \in (0, 1]$ determines a safety margin. Thus, we seek an input perturbation $\bm{\xi}^0$ having a small norm and such that $P\bm{\xi}^0 = \bm{\zeta}$ has the $i$\textsuperscript{th} and $j$\textsuperscript{th} components which satisfy the inequality
\begin{equation}
	\kappa ({z}_i + {\zeta}_i) \ge {z}_j + {\zeta}_j.
	\label{eq:zetazeta}
\end{equation}
Since this inequality involves only two components of $\bm{\zeta}$, we may restrict our further computations to these. Hence, let $P_{i,j}\in \mathbb{R}^{2\times N}$ be the submatrix of $P$ containing its $i$\textsuperscript{th} and $j$\textsuperscript{th} rows and let $\bm{\zeta}_{i,j} = (\zeta_i, \zeta_j)^{\T} \in\mathbb{R}^2$ consist of the $i$\textsuperscript{th} and $j$\textsuperscript{th} components of $\bm{\zeta}$. Then, $P_{i,j}\bm{\xi}^0 = \bm{\zeta}_{i,j}$.

\paragraph{Suppose for the time being that $\rank P_{i,j} = 2$.} In case $\bm{\zeta}_{i,j}$ is given, then the system of equations $P_{i,j}\bm{\xi}^0 = \bm{\zeta}_{i,j}$ is underdetermined and its solutions form an $(N-2)$-dimensional convex set. The \textit{least-norm solution} is given by $\bm{\xi}^0 = P_{i,j}^\T (P_{i,j} P_{i,j}^\T)^{-1} \bm{\zeta}_{i,j}$. Note that the matrix $P_{i,j} P_{i,j}^\T$ is invertible as long as $\rank P_{i,j} = 2$ since
\[
	P_{i,j} P_{i,j}^\T = \begin{pmatrix} |P_i|^2 & \langle P_i, P_j\rangle \\ \langle P_i, P_j\rangle & |P_j|^2\end{pmatrix}\quad\text{and}\quad (P_{i,j} P_{i,j}^\T)^{-1} = \frac{\tilde{P}_{j,i}\tilde{P}_{j,i}^\T}{|P_i|^2 |P_j|^2 - \langle P_i, P_j\rangle^2} =: \frac{\tilde{P}_{j,i}\tilde{P}_{j,i}^\T}{\Delta}
\]
as $|\langle P_i, P_j\rangle| < |P_i|\,|P_j|$ by the Cauchy--Schwarz inequality, where $P_i$ and $P_j$ denote the $i$\textsuperscript{th} and $j$\textsuperscript{th} row of $P$, respectively, and $\tilde{P}_{j,i} = \begin{pmatrix}  -P_j\\  P_i  \end{pmatrix}$.

We want to minimize the norm of the perturbation $\bm{\xi}^0 = P_{i,j}^\T (P_{i,j} P_{i,j}^\T)^{-1} \bm{\zeta}_{i,j}$ subject to \eqref{eq:zetazeta}. Then, $| \bm{\xi}^0|^2 = |\tilde{P}_{j,i}^\T \bm{\zeta}_{i,j}|^2/\Delta$ is a positive definite quadratic form with respect to the variable $\bm{\zeta}_{i,j} = (\zeta_i, \zeta_j)^\T$. Therefore, the minimum of $|\bm{\xi}^0|$ subject to \eqref{eq:zetazeta} is attained at the boundary of this region, where  $\zeta_j = \kappa({z}_i + {\zeta}_i) - z_j$. In other words, we shall minimize
\[
	| \bm{\xi}^0 |^2 = \frac{\left|\bigl(\kappa({z}_i + {\zeta}_i) - z_j\bigr) P_i^\T -\zeta_i P_j^\T \right|^2}{\Delta} = \frac{\left| {\zeta}_i \bigl(\kappa P_i^\T - P_j^\T\bigr) + \bigl(\kappa {z}_i - z_j\bigr) P_i^\T\right|^2}{\Delta},
\]
where $\zeta_i \in \mathbb{R}$ is the only free variable. This quantity is a second degree polynomial in $\zeta_i$ whose minimum is easy to find by, e.g., differentiation w.r.t.~$\zeta_i$. Thus, $| \bm{\xi}^0 |^2$ is minimized when
\[
	{\zeta}_i = \frac{\begin{pmatrix}1& \kappa \end{pmatrix} \tilde{P}_{j,i} \tilde{P}_{j,i}^{\mathsf{T}} \begin{pmatrix}0 \\ {z}_j - \kappa {z}_i\end{pmatrix}}{\begin{pmatrix}1&\kappa\end{pmatrix} \tilde{P}_{j,i} \tilde{P}_{j,i}^{\mathsf{T}} \begin{pmatrix}1 \\ \kappa\end{pmatrix} }.
\]
Plugging in this expression for $\zeta_i$ into the identity $\zeta_j = \kappa({z}_i + {\zeta}_i) - z_j$, we obtain by a direct calculation that the minimum occurs for $\bm{\zeta}_{i,j} = (\zeta_i, \zeta_j)^{\T}$ with
\[
	\bm{\zeta}_{i,j} = \frac{z_j - \kappa\,z_i}{\begin{pmatrix}1&\kappa\end{pmatrix} \tilde{P}_{j,i} \tilde{P}_{j,i}^{\mathsf{T}} \begin{pmatrix}1 \\ \kappa\end{pmatrix}} \, \begin{pmatrix} 0& 1\\ -1& 0\end{pmatrix}  \tilde{P}_{j,i} \tilde{P}_{j,i}^{\mathsf{T}} \begin{pmatrix}1 \\ \kappa\end{pmatrix}\,,
\]
which in turn allows us to compute and simplify
\begin{equation}
	\bm{\xi^0} = \frac{P_{i,j}^\T (\tilde{P}_{j,i} \tilde{P}_{j,i}^\T) \bm{\zeta}_{i,j} }{ \Delta} = \cdots = \frac{(z_j - \kappa z_i) (\kappa P_i^\T - P_j^\T)}{\begin{pmatrix}1&\kappa\end{pmatrix} \tilde{P}_{j,i} \tilde{P}_{j,i}^{\mathsf{T}} \begin{pmatrix}1 \\ \kappa\end{pmatrix}} = \frac{(z_j - \kappa z_i) (\kappa P_i^\T - P_j^\T)}{|\kappa P_i^\T - P_j^\T|^2}\,.
	\label{eq:minxi0}
\end{equation}
Note also that
\begin{equation}
	\label{eq:minxi0norm}    
	|\bm{\xi^0}| = \frac{z_j - \kappa z_i}{|\kappa P_i^\T - P_j^\T|}\,.
\end{equation}

\paragraph{Let us now consider the case when $\rank P_{i,j} = 1$.} We will show that even in such a case, \eqref{eq:minxi0} yields an input perturbation of least norm such that \eqref{eq:zetazeta} is satisfied, provided that $\kappa P_i \neq P_j$.

Assume first that $P_i \neq 0$ and $P_j = \lambda P_i$ for some $\lambda \in \mathbb{R}$. Similarly as before, only the limiting case when $\zeta_j = \kappa({z}_i + {\zeta}_i) - z_j$ is to be considered. The equation $P_{i,j} \bm{\xi}^0 = \bm{\zeta}_{i,j}$ becomes
\[
	\langle P_i, \bm{\xi}^0 \rangle \begin{pmatrix}1\\ \lambda \end{pmatrix}
	= \begin{pmatrix}0\\ \kappa z_i - z_j \end{pmatrix} + \zeta_i \begin{pmatrix}1\\ \kappa \end{pmatrix}, \quad \text{hence}\quad \zeta_i = \langle P_i, \bm{\xi}^0 \rangle = \frac{\kappa z_i - z_j}{\lambda - \kappa}.
\]
It follows from the Cauchy--Schwarz inequality that $\bm{\xi}^0 = \displaystyle \frac{\kappa z_i - z_j}{\lambda - \kappa} \, \frac{P_i^\T}{|P_i|^2}$ has the least norm among all solutions to $\langle P_i, \bm{\xi}^0 \rangle = (\kappa z_i - z_j)/(\lambda - \kappa)$. This particular $\bm{\xi}^0$ is also obtained by~\eqref{eq:minxi0} when one substitutes $P_j = \lambda P_i$.
Observe though that the equation $P_{i,j} \bm{\xi}^0 = \bm{\zeta}_{i,j}$ lacks any solutions if $\lambda = \kappa$. Should this occur, one might want to consider a different safety margin $\kappa \in (0, 1]$.

Assume then that $P_i = 0$ while $P_j \neq 0$. Then, the equation $P_{i,j} \bm{\xi}^0 = \bm{\zeta}_{i,j}$ becomes
\[
	\langle P_j, \bm{\xi}^0 \rangle \begin{pmatrix}0\\ 1 \end{pmatrix}
	= \begin{pmatrix}0\\ \kappa z_i - z_j \end{pmatrix} + \zeta_i \begin{pmatrix}1\\ \kappa \end{pmatrix}, \quad \text{hence}\quad \zeta_i = 0 \quad \text{and}\quad \langle P_j, \bm{\xi}^0 \rangle = \kappa z_i - z_j.
\]
Similarly as before, Cauchy--Schwarz inequality yields that $\bm{\xi}^0 = \displaystyle \frac{(\kappa z_i - z_j) \,P_j^\T}{|P_j|^2}$ has the least norm among all solutions. This particular $\bm{\xi}^0$ also corresponds to~\eqref{eq:minxi0} when $P_i = 0 \neq P_j$.

\paragraph{Finally, if $\rank P_{i,j}=0$,}{\hskip-0.5em} then our method does not produce any input perturbation $\bm{\xi}^0$ that would push the neural network output away from class $j$ towards class $i$.

\paragraph{When the input perturbation $\bm{\xi}^0$ has been found,}{\hskip-0.5em} it may happen that one cannot immediately let $\widetilde{\bm{x}}^0 = \bm{x}^0 + \bm{\xi}^0$. One possible issue is that the perturbed input $\widetilde{\bm{x}}^0$ has values that are invalid. For example, when classifying images with pixel values normalized to lie in the interval $[0, 1]$, it can happen that $\widetilde{\bm{x}}^0$ would contain pixels with a negative value or a value exceeding $1$. Another possible issue stems from the approximation error in $\widetilde{\bm{x}}(T) \approx \bm{x}(T) + \bm{\xi}(T)$, which may be too large in case the magnitude of $\bm{\xi}^0$ is too large.

To prevent these issues, the adversarial attacks could be generated in an iterative process, where the clean input $\bm{x}^0$ is perturbed by taking only a short step in the direction of $\bm{\xi}^0$ and thereby producing $\bm{x}^1$ in the vicinity of $\bm{x}^0$. Should $\bm{x}^1$ lie outside of the region with valid values, then it would be replaced by a nearest point in the valid region. Since the inferred class of $\bm{x}^1$ may coincide with the class of $\bm{x}^0$, one applies the method above to find a perturbation $\bm{\xi^1}$ for $\bm{x}^1$ and take only a short step from there in the direction of $\bm{\xi}^1$, thereby generating $\bm{x}^2$. The process is iterated until a misclassification occurs or until a maximum number of allowed iterations is reached.
\subsection{Untargeted attack via targeted attack with a moving target}
Similarly as above, we assume that the neural network classifies the input $\bm{x}^0$ as class $j$. We search for a (small) perturbation such that the perturbed input would be classified differently, without any particular preference as to which class it would be.

Similarly as above, this attack is generated by an iterative process, where we determine in each iteration separately which class $i$ is to be preferred. The value of $i$ is chosen so that it minimizes the norm of $\bm{\xi}^0$ in \eqref{eq:minxi0norm}. Observe that if $z_j \le \kappa z_i$, then we have already found an adversarial attack and the iteration may be terminated. Thus, only a minor modification of the algorithm suffices to ensure that the target is updated in each iteration.

\subsection{Summary and results}
\label{sec:attackSum}
For targeted attacks as described in Subsection~\ref{ssec:tgtAttacks}, an input perturbation is sought so that class $i$ is given preference over class $j$. Note however that the method lacks any control over other remaining classes. It is therefore very well possible that another class would overpower both classes $i$ and $j$ in the network output.

\begin{algorithm}[ht]
	\caption{Adversarial attacks, both targeted and untargeted (via moving target)}
	\label{alg:advAttack}
	\begin{algorithmic}[1]
		\renewcommand{\algorithmicrequire}{\textbf{Input:}}
		\renewcommand{\algorithmicensure}{\textbf{Output:}}
		\REQUIRE An sNODE neural network (the algorithm can be easily adapted to ResNets though)
		\REQUIRE An input datapoint $\bm{x}^0$ whose inferred class is $j$.
		\REQUIRE [\emph{For targeted attacks only:} Target class $i$, which is to be given preference.]
		\REQUIRE Safety margin $\kappa \in (0, 1]$, step length $s>0$ and max.~number of iterations $M>0$
		\ENSURE A point $\bm{x}$ in the vicinity of $\bm{x}^0$ whose inferred class differs from $j$
		\STATE Solve the direct problem~\ref{eq:xDE}) with $\bm{x}(0) = \bm{x}^0$.
		\FOR {$m=0,1,2\ldots, M-1$} 
		\STATE Solve the neural sensitivity problem~\eqref{eq:spAdv} with $\bm{\xi}(0) = I_{N\times N}$.
		\STATE Let $P = \bm{\xi}(T)$.
		\STATE [\emph{For untargeted attacks only:} Find $i$ that minimizes \eqref{eq:minxi0norm}.]
		\STATE Compute $\bm{\xi}^{m}$ as in \eqref{eq:xi0-attack}.
		\STATE Let $\bm{\tilde{x}}^{m+1} = \bm{x}^m + s\,\bm{\xi}^{m} / |\bm{\xi}^{m}|$, and $\bm{x}^{m+1}$ be the point in the valid region nearest to $\bm{\tilde{x}}^{m+1}$
		\STATE Solve~\ref{eq:xDE}) with $\bm{x}(0) = \bm{x}^{m+1}$.
		\IF{$\arg\max \bm{x}(T) \neq j$}
		\RETURN $\bm{x}^{m+1}$
		\ENDIF
		\ENDFOR
	\end{algorithmic} 
\end{algorithm}

We compare the proposed attack method to projected gradient descent (PGD)~\citep{madry2018towards}, the fast gradient signed method (FGSM)~\citep{goodfellow2014explaining}, and the method proposed by \cite{carlini2017towards} which we refer to as CW. Our method is tuned to run $15$ iterations with a step length of $0.06$. We compare to the same settings with PGD, but also note that this method can benefit from larger step length. Thus, we also compare to PGD with step length $0.6$.

Figure~\ref{fig:adv_pgd} shows comparisons in terms of success rate of adversarial attacks for different amount of allowed perturbation of the input images. Comparisons have been performed on the test set of the downsampled MNIST (14x14) dataset, for various combinations of training method (SGD, NCG, Sobolev NCG), activation function (ReLU, tanh) and loss functions (CE, $\ell^2$). From the results, we can see that the adversarial attack with a moving target has a success rate on par with the attack where each image is subjected to nine different targeted attacks, yet it runs approximately nine times faster. Compared to PGD with the same settings, our method is consistently more successful. However, PGD with a larger step length outperforms our method for most NCG trained models.
For CW, it is difficult to control the maximum norm of adversarial perturbations. For our models, we also notice how this method uses large perturbations for the successful attacks, which means that the perturbation to success ratio is substantially larger compared to our method and PGD, see Table~\ref{tab:cw}.

In our experiments, we have found an image in the downsampled MNIST dataset that is susceptible to our adversarial attacks (both targeted and untargeted) as well as PGD attacks, in ResNet/sNODE neural networks produced by various training methods (SGD, NCG, and Sobolev NCG) so that the input perturbations have a similar $\ell^2$-norm in all these cases. In Figure~\ref{fig:5underAttack}, we present the original image as well as its misclassified attacked versions with the corresponding adversarial perturbations. The training regime used to find weights of the neural network has evidently a significant effect on the form of the adversarial perturbation.

Some further example images of adversarial attacks are demonstrated in Figures~\ref{fig:adv_ex_sgd} and~\ref{fig:adv_ex_sob}. As can be seen, the successful attacks with CW use large perturbations. Comparing the model trained with SGD, in Figure~\ref{fig:adv_ex_sgd}, to the model trained with Sobolev NCG, in Figure~\ref{fig:adv_ex_sob}, we observe how the adversarial attacks of the latter in general rely to a lesser extent on single pixels with large perturbation. That is, it seems that an SGD trained model can be attacked by large perturbations of single pixels, while the NCG trained model is more resilient to such attacks, which to some extent could explain the increased robustness to adversarial attacks of NCG.

\clearpage
\begin{figure}[t]
	\begin{centering}
		\begin{subfigure}{\textwidth}
			\includegraphics[width=0.328\linewidth,trim=2mm 0 4mm 0, clip]{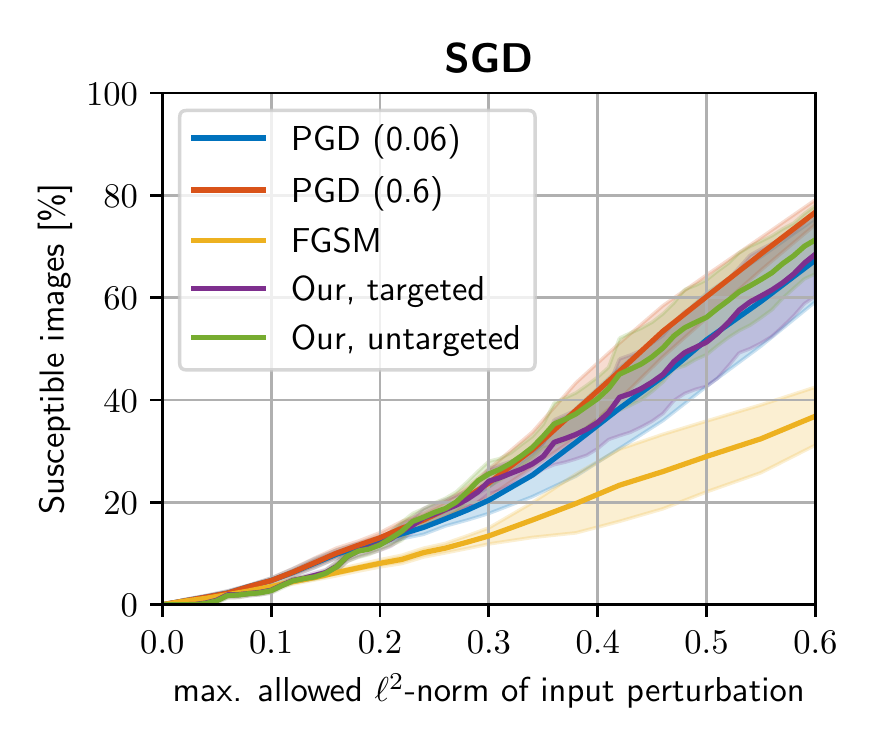} 
			\includegraphics[width=0.328\linewidth,trim=2mm 0 4mm 0, clip]{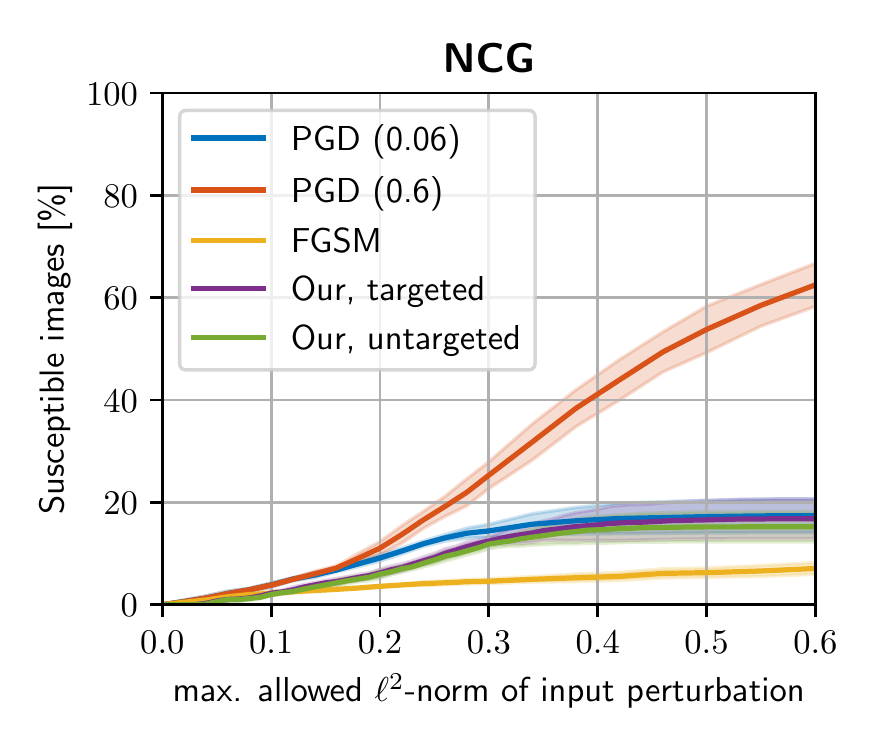} 
			\includegraphics[width=0.328\linewidth,trim=2mm 0 4mm 0, clip]{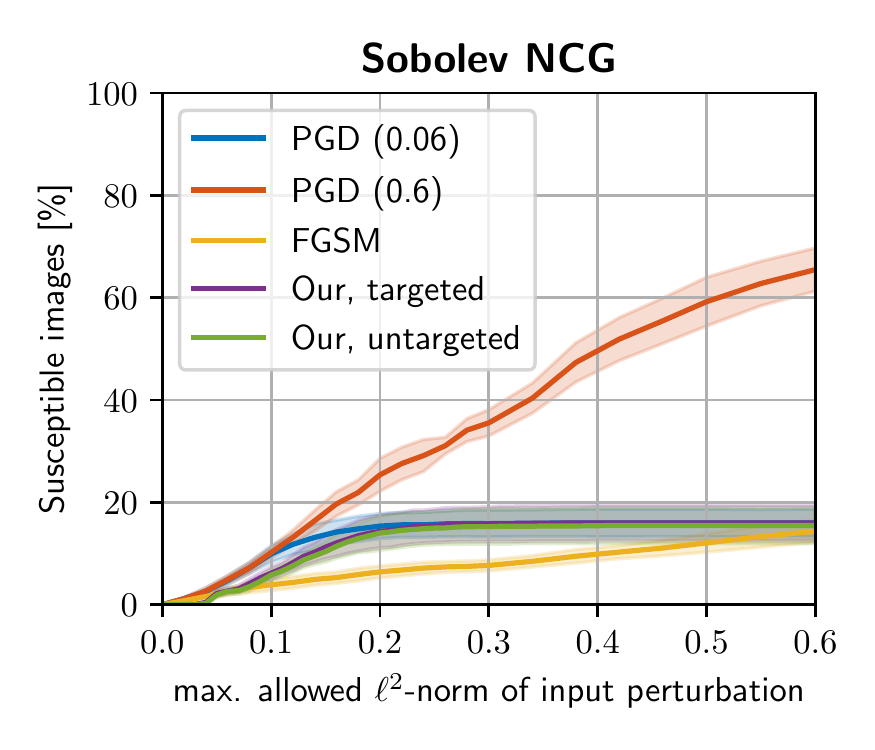} 
			\caption{ReLU activation, CE-loss}
		\end{subfigure}
		    
		\begin{subfigure}{\textwidth}
			\includegraphics[width=0.328\linewidth,trim=2mm 0 4mm 0, clip]{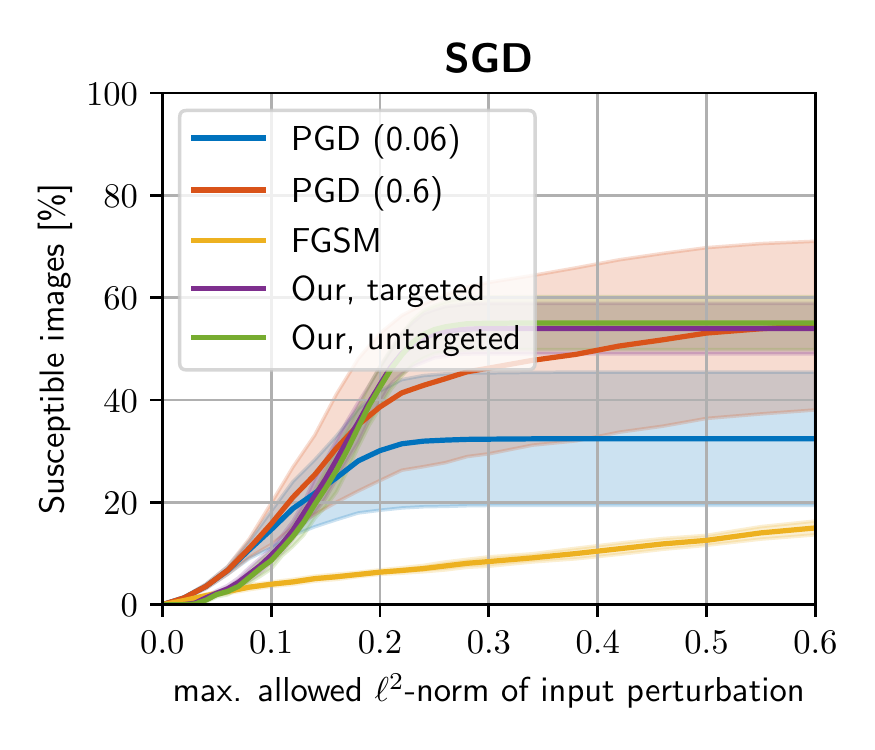} 
			\includegraphics[width=0.328\linewidth,trim=2mm 0 4mm 0, clip]{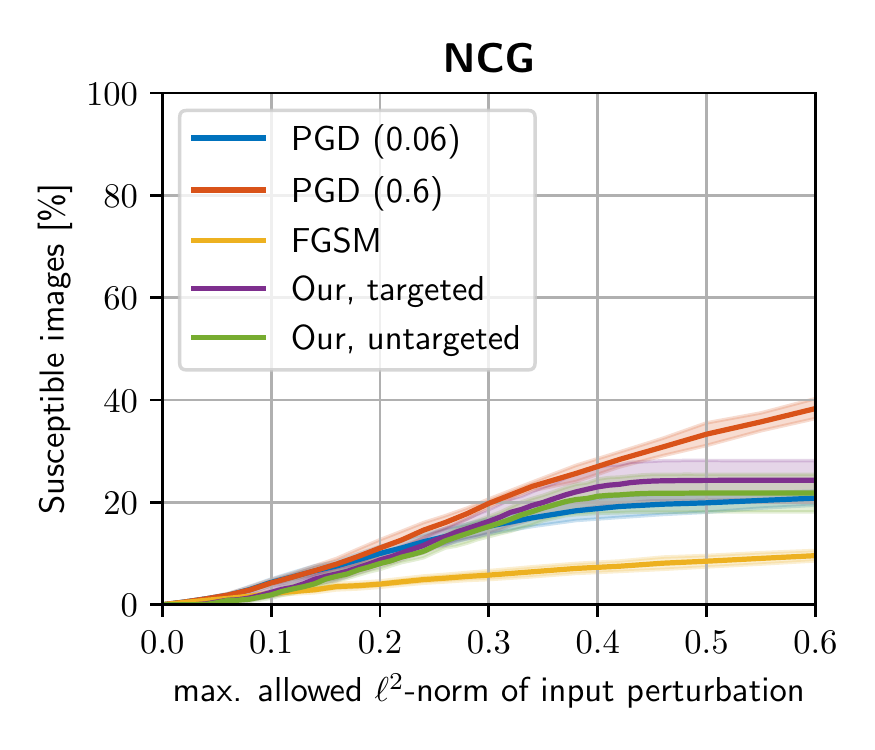} 
			\includegraphics[width=0.328\linewidth,trim=2mm 0 4mm 0, clip]{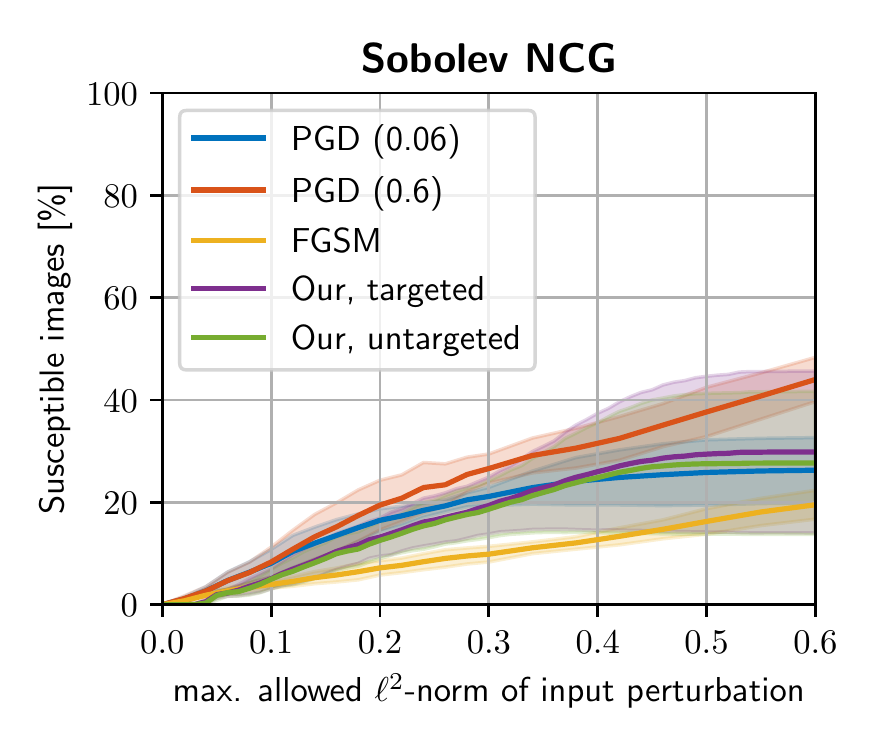} 
			\caption{ReLU activation, $\ell^2$-loss}
		\end{subfigure}
		    
		\begin{subfigure}{\textwidth}
			\includegraphics[width=0.328\linewidth,trim=2mm 0 4mm 0, clip]{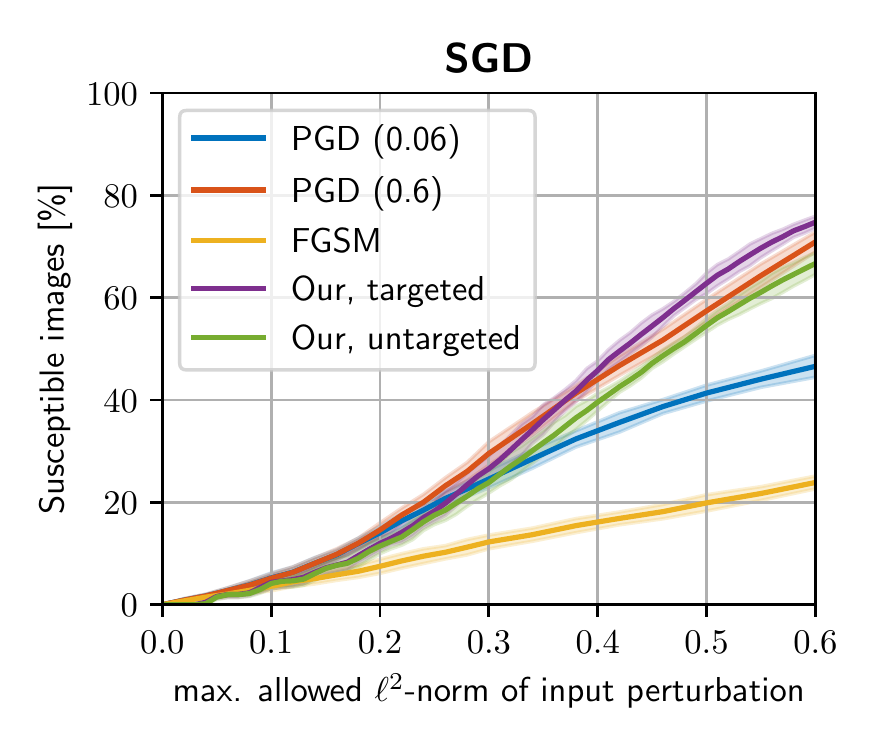} 
			\includegraphics[width=0.328\linewidth,trim=2mm 0 4mm 0, clip]{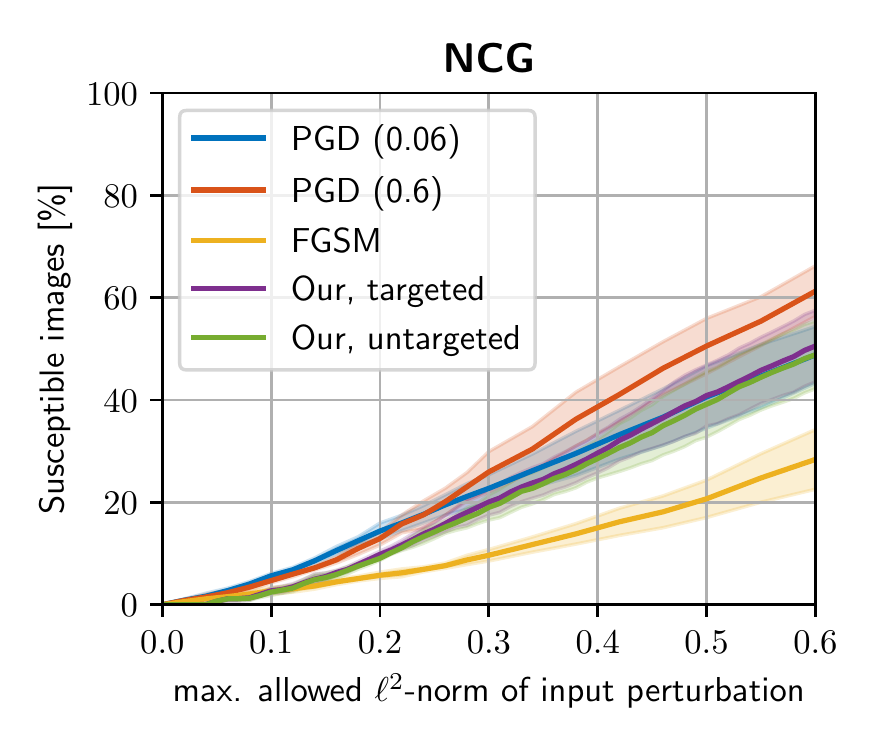} 
			\includegraphics[width=0.328\linewidth,trim=2mm 0 4mm 0, clip]{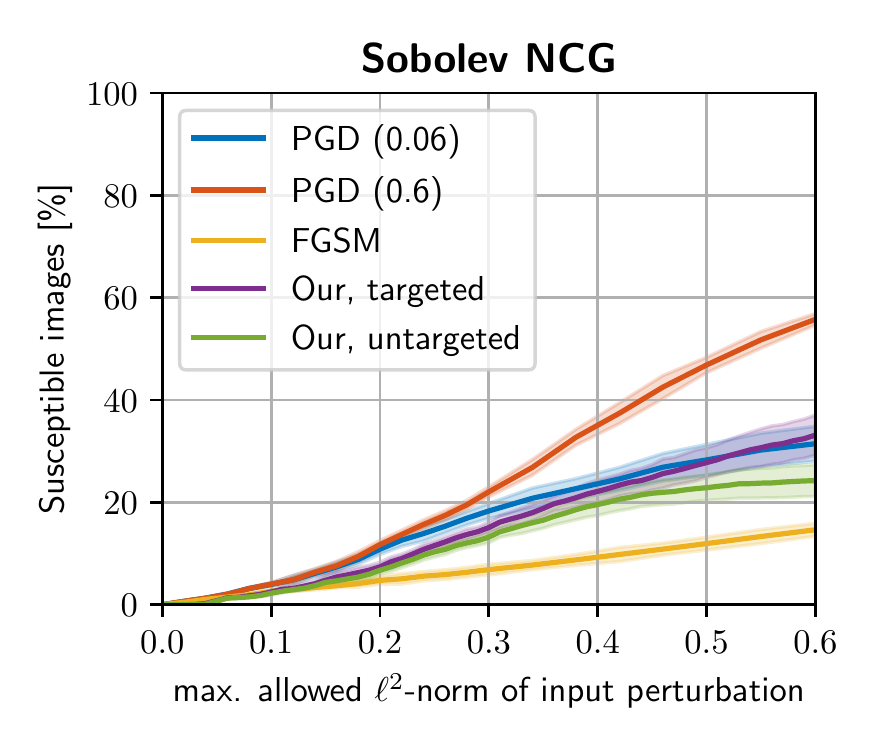} 
			\caption{tanh activation, CE-loss}
		\end{subfigure}
		    
		\begin{subfigure}{\textwidth}
			\includegraphics[width=0.328\linewidth,trim=2mm 0 4mm 0, clip]{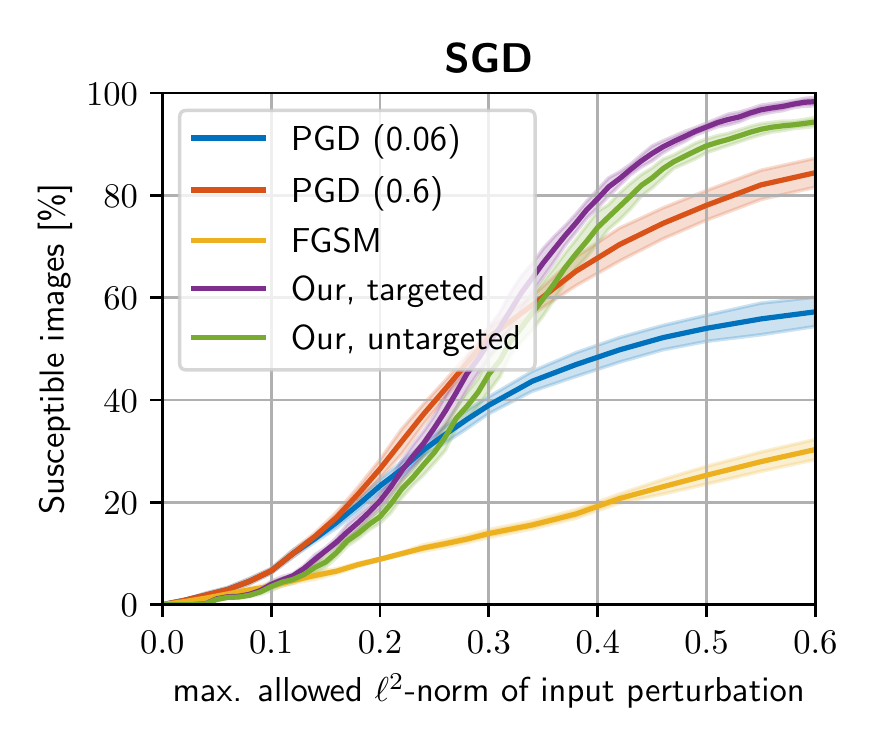} 
			\includegraphics[width=0.328\linewidth,trim=2mm 0 4mm 0, clip]{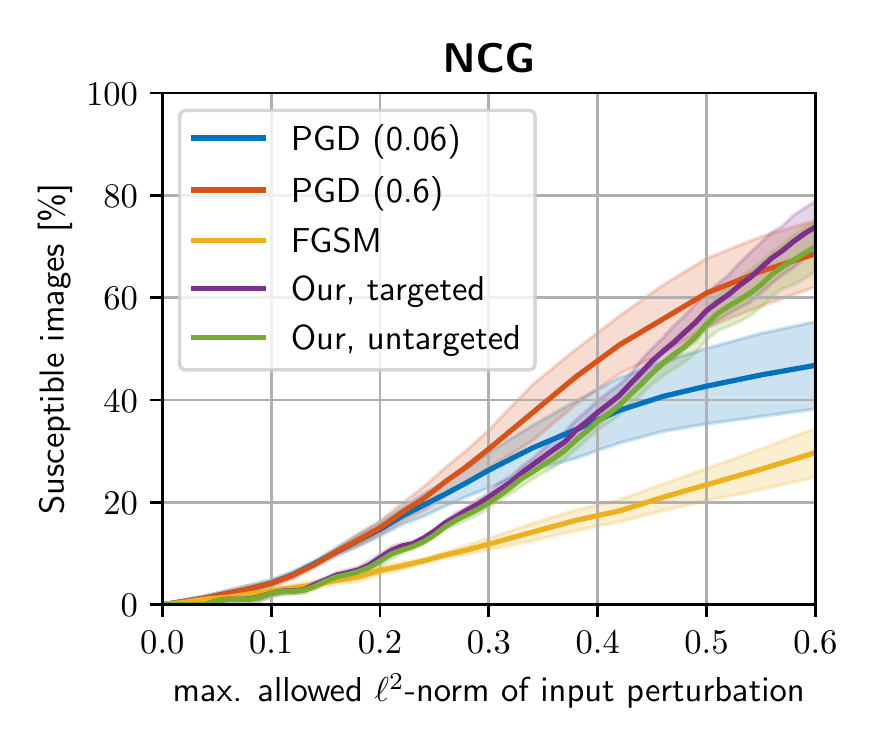} 
			\includegraphics[width=0.328\linewidth,trim=2mm 0 4mm 0, clip]{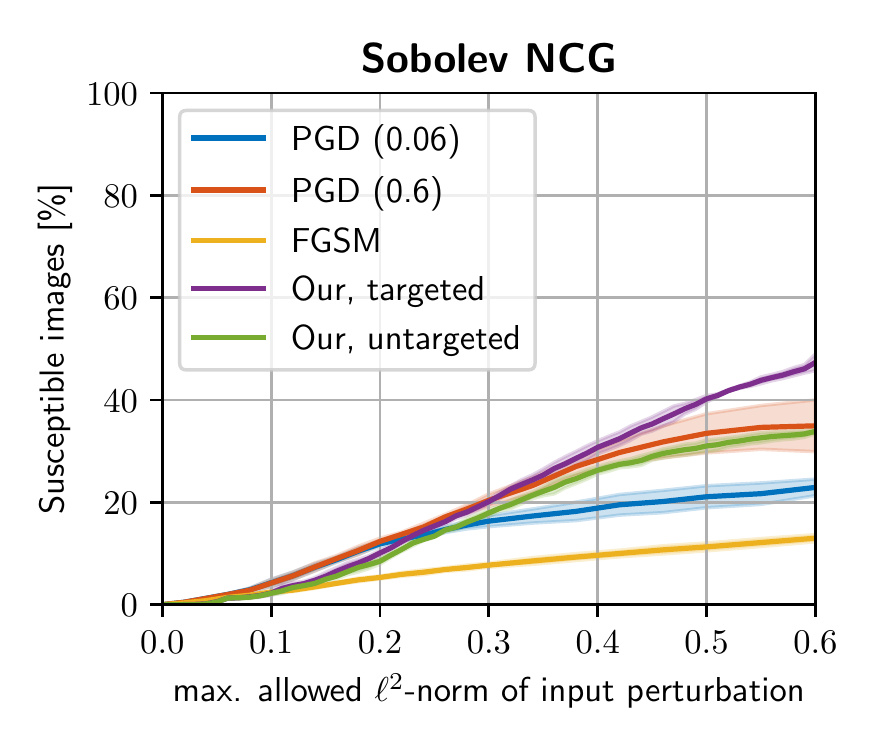} 
			\caption{tanh activation, $\ell^2$-loss}
		\end{subfigure}
		\caption{Comparison of different adversarial attacks on MNIST 14x14, for different training methods (columns) and different combinations of activation function and loss (rows).}
		\label{fig:adv_pgd}
	\end{centering}
	\vspace{-5mm}
\end{figure}
\clearpage
\phantom{.}
\vfill
\begin{table}[ht]
	\centering
	\caption{Percentage of MNIST 14x14 test set that are susceptible to CW attacks. The perturbations specify the mean and maximum $\ell^2$-norms of the successful attacks. Numbers also include standard deviations over 5 separate training runs.}
	\begin{tabular}{ll|llll}
		                         &                          & ReLU, CE         & ReLU, $\ell^2$    & tanh, CE         & tanh, $\ell^2$   \\
		\hline
		\multirow{3}{*}{SGD}     & Susceptible [\%]         & $42.83 \pm 2.83$ & $52.74 \pm 1.03$  & $27.87 \pm 1.36$ & $56.44 \pm 2.66$ \\
		                         & Mean perturb. [$\ell^2$] & $0.76 \pm 0.02$  & $0.96 \pm 0.01$   & $0.45 \pm 0.02$  & $0.88 \pm 0.02$  \\
		                         & Max perturb. [$\ell^2$]  & $2.35 \pm 0.19$  & $2.06 \pm 0.05$   & $1.64 \pm 0.45$  & $2.28 \pm 0.20$  \\
		         
		\hline
		\multirow{3}{*}{NCG}     & Susceptible [\%]         & $3.78 \pm 0.19$  & $26.56 \pm 2.70$  & $3.52 \pm 0.67$  & $28.15 \pm 2.01$ \\
		                         & Mean perturb. [$\ell^2$] & $0.13 \pm 0.01$  & $0.75 \pm 0.02$   & $0.09 \pm 0.01$  & $0.66 \pm 0.03$  \\
		                         & Max perturb. [$\ell^2$]  & $0.28 \pm 0.02$  & $1.99 \pm 0.12$   & $0.21 \pm 0.02$  & $2.13 \pm 0.19$  \\
		         
		\hline
		\multirow{3}{*}{Sobolev} & Susceptible [\%]         & $5.81 \pm 0.85$  & $37.45 \pm 16.17$ & $4.17 \pm 0.48$  & $54.91 \pm 1.48$ \\
		                         & Mean perturb. [$\ell^2$] & $0.13 \pm 0.01$  & $0.75 \pm 0.13$   & $0.11 \pm 0.00$  & $0.75 \pm 0.02$  \\
		                         & Max perturb. [$\ell^2$]  & $0.25 \pm 0.01$  & $2.39 \pm 0.36$   & $0.24 \pm 0.02$  & $2.51 \pm 0.30$  \\
	\end{tabular}
	\label{tab:cw}
\end{table}
\vfill
\begin{figure}[ht]
	\includegraphics[width=0.99\textwidth]{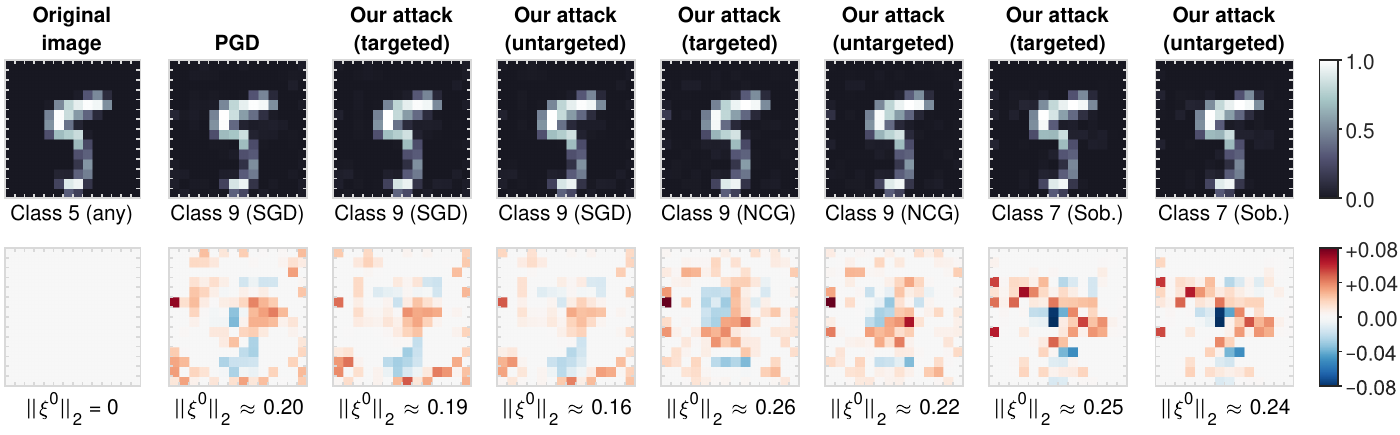}
	\caption{An image whose classification can be changed in several neural networks by adversarial attacks with input perturbations of similar $\ell^2$-norm. These ResNet/sNODE networks were trained with $\tanh$ as an activation function and CE-loss without weight decay. The first row depicts the affected input data $\widetilde{\bm{x}}^0 = \bm{x}^0 + \bm{\xi}^0$. The second row shows the perturbation $\bm{\xi}^0$ in each respective case.}
	\label{fig:5underAttack}
\end{figure}
\vfill
\pagebreak
\begin{figure}
	\centering
	\begin{subfigure}{\textwidth}
		\includegraphics[width=\linewidth,trim=25mm 10mm 5mm 10mm, clip]{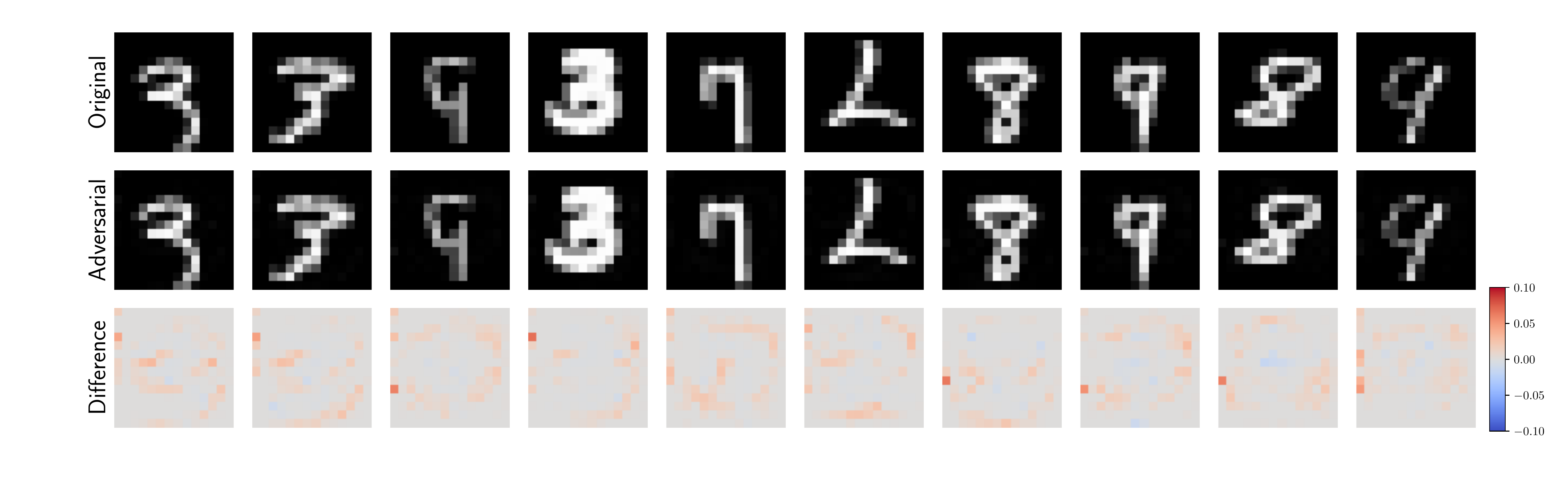}
		\caption{PGD, $\ell^2 = 0.1$}
	\end{subfigure}
	\vspace{3mm}
	    
	\begin{subfigure}{\textwidth}
		\includegraphics[width=\linewidth,trim=25mm 10mm 5mm 10mm, clip]{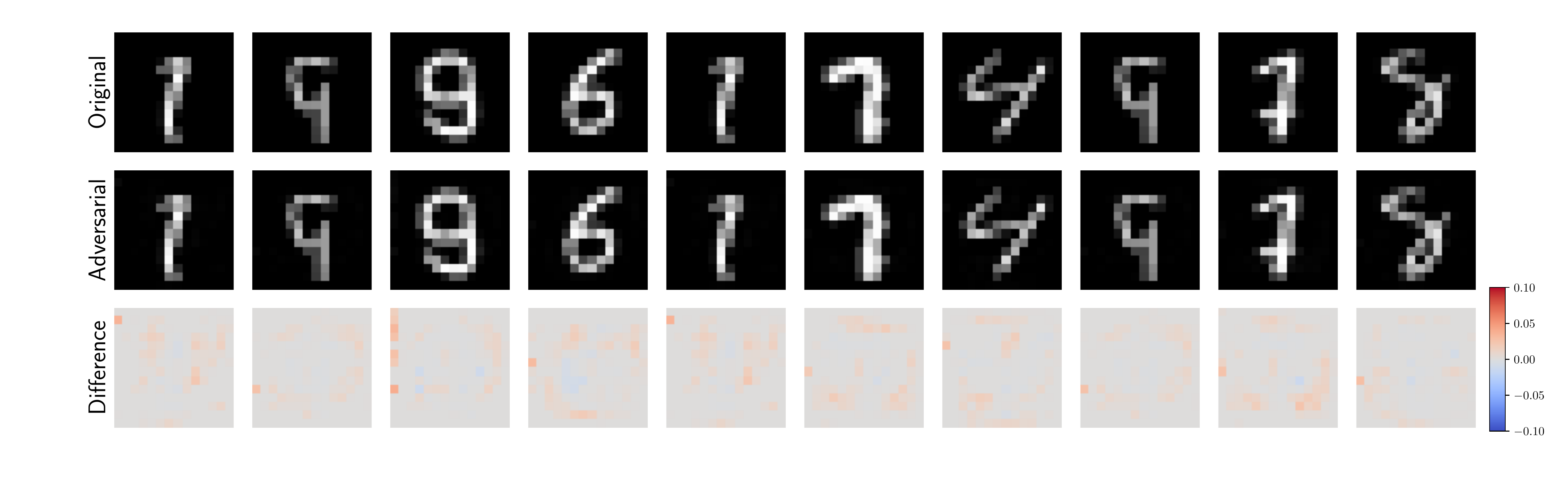}
		\caption{FGSM, $\ell^2 = 0.1$}
	\end{subfigure}
	\vspace{3mm}
	    
	\begin{subfigure}{\textwidth}
		\includegraphics[width=\linewidth,trim=25mm 10mm 5mm 10mm, clip]{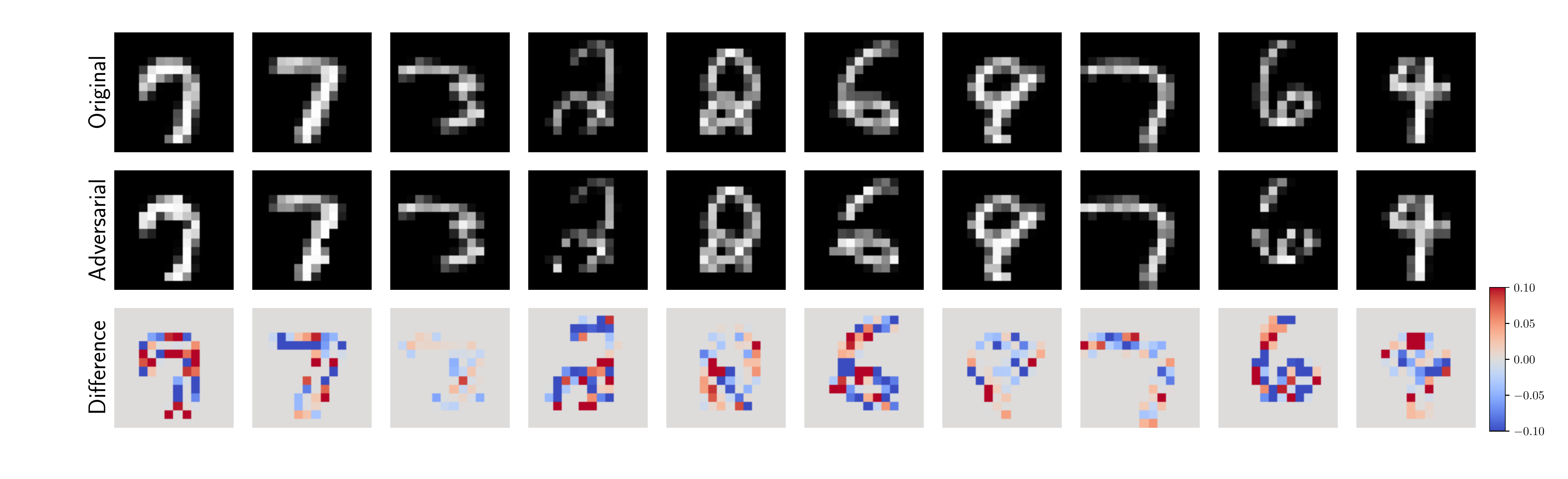}
		\caption{CW, $\ell^2 \approx 0.8$}
	\end{subfigure}
	\vspace{3mm}
	    
	\begin{subfigure}{\textwidth}
		\includegraphics[width=\linewidth,trim=25mm 10mm 5mm 10mm, clip]{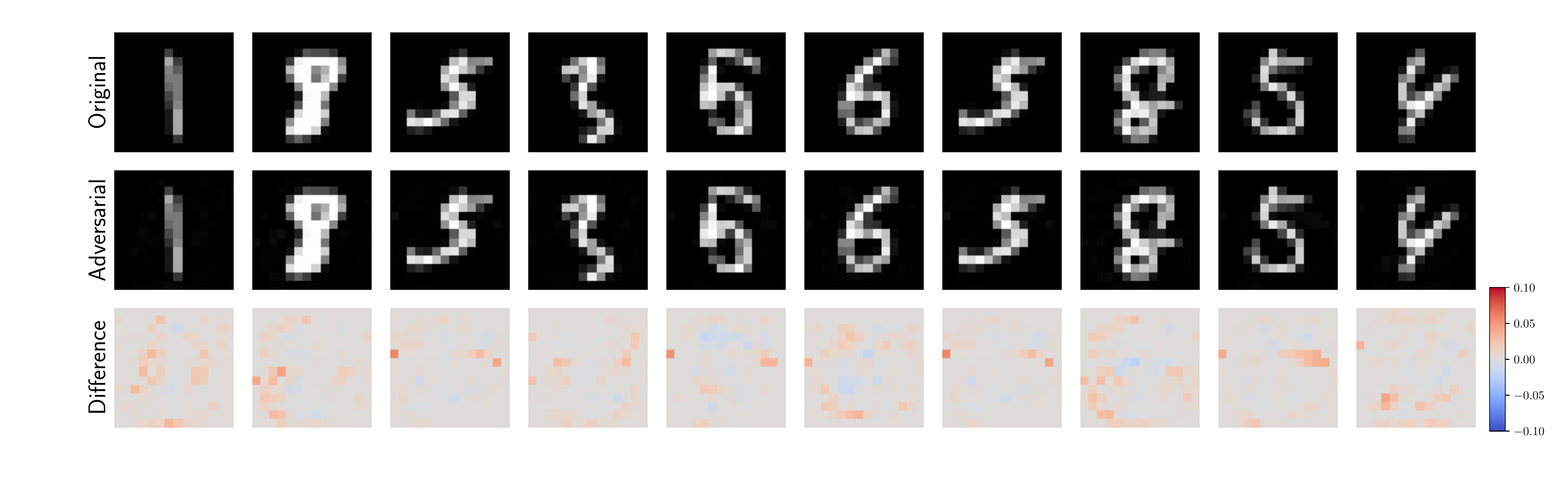}
		\caption{Our, $\ell^2 = 0.1$}
	\end{subfigure}
	\caption{Randomly selected examples of successful attacks, for an SGD trained model with ReLU activation and $\ell^2$ loss trained on MNIST 14x14. For PGD, FGSM, and our method, the adversarial perturbations have an $\ell^2$-norm of $0.1$. For the CW attack, the norms of successful attacks are larger.}
	\label{fig:adv_ex_sgd}
\end{figure}
\clearpage
\begin{figure}
	\centering
	\begin{subfigure}{\textwidth}
		\includegraphics[width=\linewidth,trim=25mm 10mm 5mm 10mm, clip]{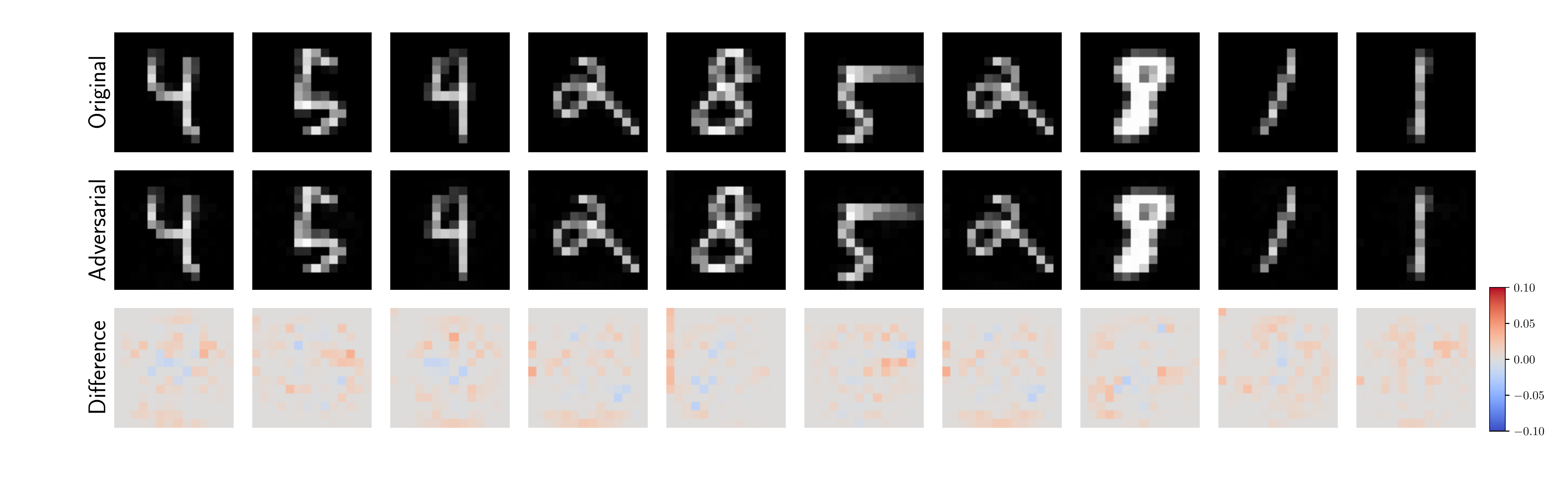}
		\caption{PGD, $\ell^2 = 0.1$}
	\end{subfigure}
	\vspace{3mm}
	    
	\begin{subfigure}{\textwidth}
		\includegraphics[width=\linewidth,trim=25mm 10mm 5mm 10mm, clip]{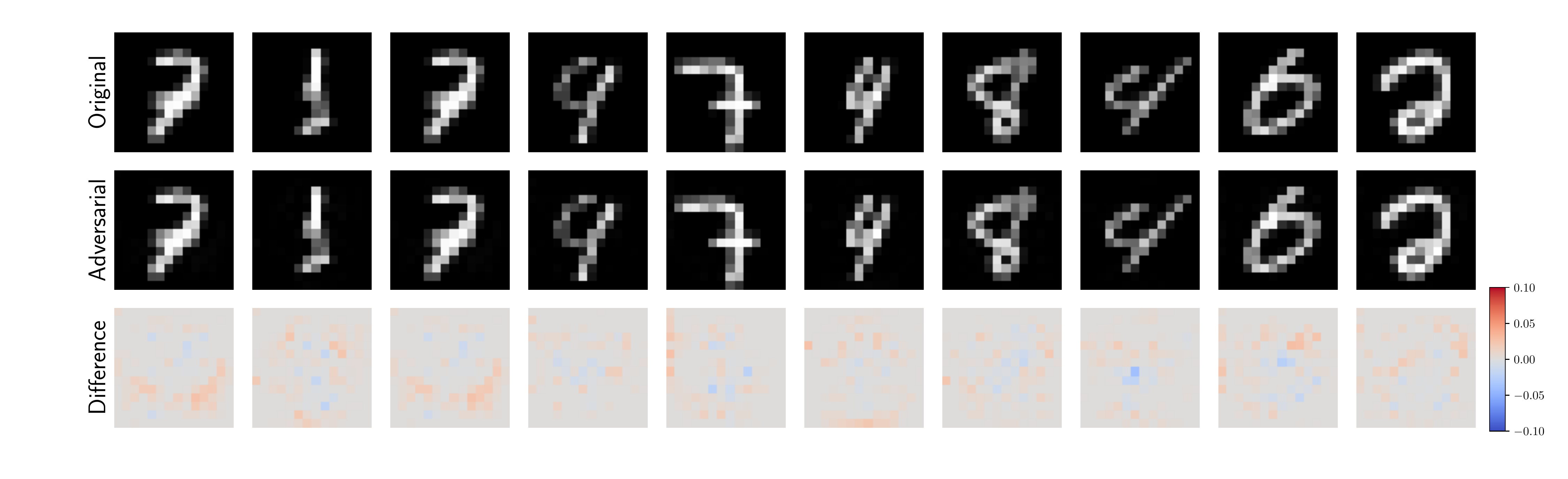}
		\caption{FGSM, $\ell^2 = 0.1$}
	\end{subfigure}
	\vspace{3mm}
	    
	\begin{subfigure}{\textwidth}
		\includegraphics[width=\linewidth,trim=25mm 10mm 5mm 10mm, clip]{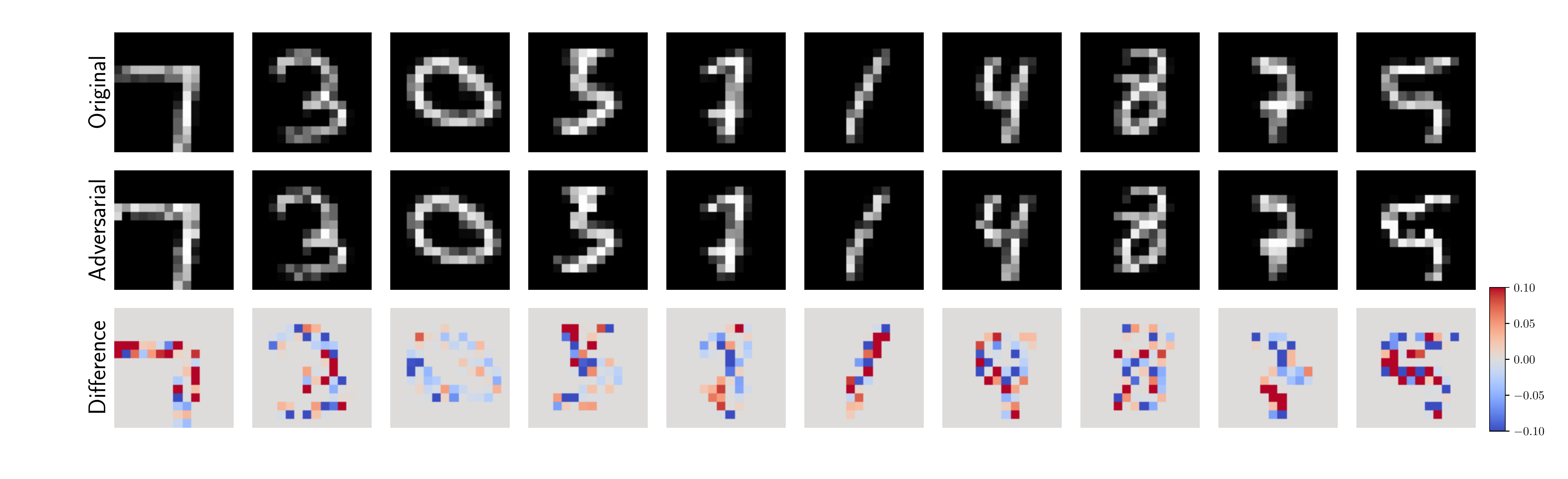}
		\caption{CW, $\ell^2 \approx 0.8$}
	\end{subfigure}
	\vspace{3mm}
	    
	\begin{subfigure}{\textwidth}
		\includegraphics[width=\linewidth,trim=25mm 10mm 5mm 10mm, clip]{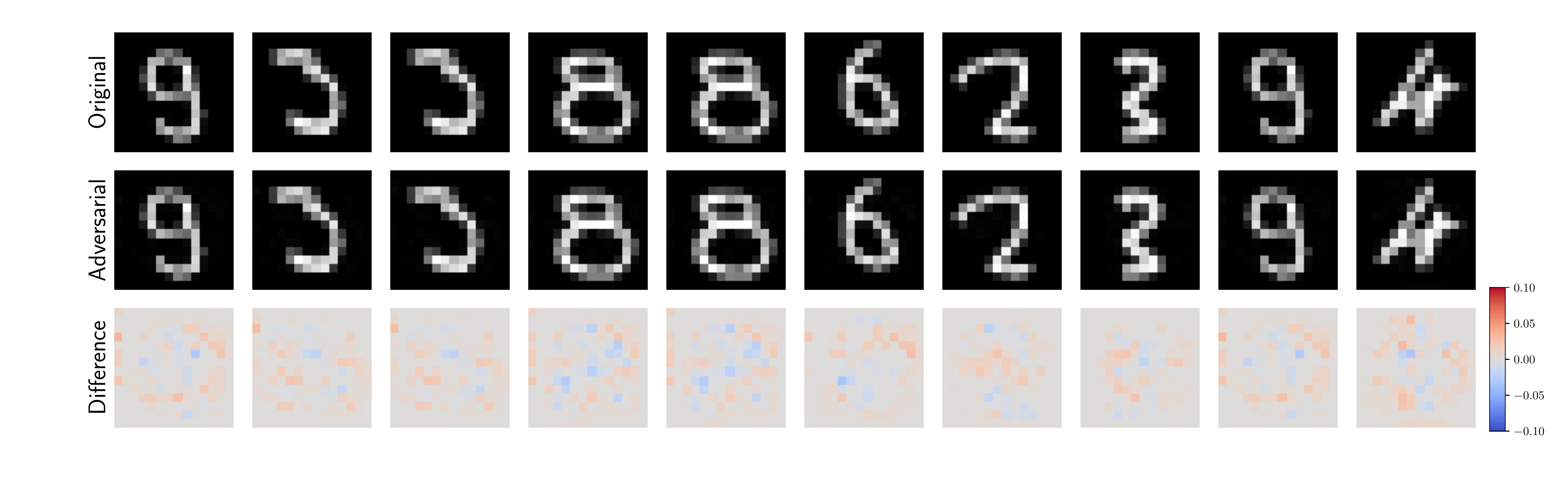}
		\caption{Our, $\ell^2 = 0.1$}
	\end{subfigure}
	\caption{Same as Figure~\ref{fig:adv_ex_sgd}, but for a model trained with Sobolev NCG.}
	\label{fig:adv_ex_sob}
\end{figure}
\end{document}